\documentclass[lettersize,journal]{IEEEtran}
\usepackage{amsmath,amsfonts}
\usepackage{algorithmic}
\usepackage{algorithm}
\usepackage{array}
\usepackage[caption=false,font=normalsize,labelfont=sf,textfont=sf]{subfig}
\usepackage{textcomp}
\usepackage{stfloats}
\usepackage{url}
\usepackage{verbatim}
\usepackage{graphicx}
\usepackage{xcolor}
\usepackage{cite}
\usepackage{flushend}
\usepackage{enumitem}
\definecolor{darkgreen}{rgb}{0,0.6,0}
\usepackage[bookmarks=true,colorlinks=true,pdfpagemode=UseNone,citecolor=darkgreen,linkcolor=blue,urlcolor=blue]{hyperref}
\newtheorem{theorem}{Theorem}
\newtheorem{lemma}{Lemma}

\usepackage{booktabs} 
\usepackage{multirow}
\usepackage{mathtools}

\begin{document}

\title{Adaptive Rejection of Unknown Hybrid Sinusoidal Disturbances with Application to Humanoid Walking under Unknown Ground Motion}


\author{I-Chia Chang$^{1*}$, Jacob Stewart$^{2*}$, Petros A. Ioannou$^{2}$, and Yan Gu$^{1\dagger}$ 
\thanks{
$^{*}$These authors contributed equally. 
$^{^{\dagger}}$Corresponding author.}
\thanks{$^{1}$I-Chia Chang and Yan Gu are with the School of Mechanical Engineering, Purdue University, West Lafayette, IN 47907, USA. E-mails: {\tt \footnotesize \{chang970, yangu\}@purdue.edu.}}
\thanks{$^{2}$Jacob Stewart and Petros A. Ioannou are with the Viterbi School of Engineering, University of Southern California, Los Angeles, CA 90007, USA. E-mails: {\tt \footnotesize \{jacobmst, ioannou\}@usc.edu.}}%

\vspace{-0.15in}
}
\markboth{Journal of \LaTeX\ Class Files,~Vol.~14, No.~8, August~2021}%
{Shell \MakeLowercase{\textit{et al.}}: A Sample Article Using IEEEtran.cls for IEEE Journals}


\maketitle

\begin{abstract}
Legged robot control is not yet practical in non-inertial environments, where the ground and surroundings accelerate in the inertial Earth frame, mainly because of the unknown, time-varying, persistent perturbations caused by environmental motion, compounded by the associated hybrid robotic dynamics.
We formulate the problem of robot walking control as unknown-disturbance rejection for hybrid linear systems subject to hybrid sinusoidal disturbances (i.e., a linear combination of sinusoidal components with instantaneous phase jumps at discrete events).
First, necessary and sufficient conditions for solvability of the disturbance rejection problem are derived and used to characterize the control input that rejects the disturbances.
Second, with the Youla-Kucera parametrization, an output feedback controller capable of rejecting the disturbance is introduced.
Third, an adaptive law is proposed to estimate the controller parameters that reject the disturbance in the practical case where the disturbance and its parameters are unknown.
Experiments on bipedal walking during ground motion demonstrate substantially improved body-position tracking accuracy, with an average improvement of over $60\%$ in simulations and up to $47\%$ in real-world experiments. 
\end{abstract}


\vspace{-0.1in}
\section{Introduction}

Legged locomotion in non-inertial environments, where the ground accelerates relative to an inertial Earth frame, remains an open challenge as the ground motion induces unknown, time-varying, persistent perturbations to the complex robot dynamics~\cite{chang2026survey, misenti_experimental_2025, gao2024ALIPDRS}.
Legged robot dynamics associated with locomotion can be modeled as a hybrid, nonlinear dynamical system, with the shifts in the support-foot position and the body velocity, upon a foot landing, formulated as discrete jumps~\cite{xiong_3-d_2022, westervelt_feedback_2018, gu2026evolution}.
Additionally, surface motion acts as an unknown, time-varying exogenous disturbance applied to the robot's supporting leg that continuously perturbs the system, rendering legged locomotion in non-inertial environments an unknown, time-varying disturbance rejection problem for hybrid dynamical systems.

By explicitly including non-inertial environments during control policy training, data-driven methods have enabled quadruped skateboarding~\cite{da2021learning, liu_discrete-time_2025} and balancing on diverse moving platforms~\cite{ma_dreureka_2024, yoon_learning-based_2024, zinage2026contractionppo}.
However, data-driven methods typically cannot provide theoretical performance guarantees, and the tasks considered primarily focus on quadrupedal locomotion, which is inherently more stable than bipedal locomotion.
Prior feedback and predictive control literature typically utilizes reduced-order models to ensure real-time computation and simplify the challenge by assuming full ground-motion knowledge~\cite{iqbal2020provably, gao2024ALIPDRS}, treating environment dynamics as simple bounded input perturbations~\cite{chen2024contingency}, or transforming the hybrid error dynamics to a discrete or continuous error system~\cite{stewart2025adaptive,gao2024ALIPDRS}. 
In this paper, we directly handle the hybrid disturbance rejection problem with linear reduced-order models assuming the disturbance is the sum of sinusoids with unknown frequencies, amplitudes, and phases, which is an effective model for platform motion such as ship or wave-induced motion~\cite{chang2026survey}.

Although the controller design for hybrid dynamical systems has been widely studied, e.g., output regulation~\cite{carnevale2015hybrid}, output feedback stabilization~\cite{possieri2016structural}, and optimal tracking~\cite{possieri2022optimal}, the output feedback control design for hybrid linear systems with unknown and unmatched sinusoidal disturbances for both continuous and discrete phases remains open.
For example, \cite{Heemels2016L2analysis, possieri2017mathcal} study the $\mathcal{L}_2$-gain from the disturbance with unknown structure to the output of a hybrid linear system.
\cite{l2023model} proposes a model reference adaptive controller for nonlinear, time-varying, hybrid systems affected by matched and parametric uncertainties.
\cite{zhao2024adaptive} suppresses the disturbance for switched systems using a disturbance observer and state feedback.
In contrast, this paper aims to achieve unknown disturbance rejection with only output feedback through Youla-Kucera (YK) parameterization~\cite{kucera2011youla,landau2020youla}.

The YK parameterization~\cite{kucera2011youla,landau2020youla} describes the family of all stabilizing controllers for a linear time-invariant system in terms of a stable linear operator whose parameters may vary without destabilizing the closed-loop system.
The capability to ensure closed-loop stability is critical when designing an adaptive controller as the time variation of controller parameters is inevitable during adaptation.
Robust adaptive controllers exploit the ability to vary the parameters by using the swapping lemma~\cite{ioannou1993ltvSystems, Ioannou2006tutorial} to generate a linear estimation equation whose optimal parameters result in a controller that makes the output zero \cite{ioannou2012robust,Ioannou2006tutorial,ioannou1993ltvSystems}.
Combined, the YK structure and the swapping lemma offer a systematic framework for designing robust adaptive controllers for linear systems with guaranteed stability and robustness to noise, model uncertainty, and time-varying parameters \cite{jafari2015robust,jafari2016contMimo,jafari2016discMimo,jafari2017contSiso,jafari2017overparameterized,stewart2024ltiSystem,stewart2024ltvSystem,stewart2025adaptive,stewart2026mimo}.
However, extending the observer-compensator structure of the YK parameterization to a hybrid linear system is not straightforward as the jump dynamics consistently inject time-varying transition effects into the system.

In this paper, motivated by bipedal locomotion in non-inertial environments, we analytically study disturbance rejection for hybrid linear dynamical systems subject to hybrid sinusoidal disturbances.
Our prior work~\cite{stewart2025adaptive} addressed the same locomotion control problem by assuming zero swing foot tracking error to transform the hybrid error system into a continuous one, enabling the use of continuous-system adaptive controllers but validated only in simulation on a five-link robot in MATLAB.
Here, we instead develop an adaptive controller directly for the hybrid linear system and integrate it into a proposed hierarchical control framework.
The new contributions of this work are:
\begin{enumerate}
    

    \item We derive \emph{Hybrid Regulator Equations} that give necessary and sufficient solvability conditions for the disturbance rejection problem in a class of hybrid linear systems. Their solutions are constant matrices, which makes them more tractable than the equations in~\cite{carnevale2015hybrid}, whose solutions are time-varying matrices.

    
    

    \item We develop an adaptive output feedback controller for hybrid linear systems with unknown and unmatched hybrid sinusoidal disturbances. This is enabled by (i) extending the observer-compensator structure of the YK parameterization to hybrid linear systems by canceling the time-varying transition effect using the hybrid structure (see Theorem~\ref{The: ctrl design}), and (ii) deriving the \emph{Hybrid Swapping Lemma}, which extends the continuous and discrete Swapping Lemma~\cite{ioannou1993ltvSystems, Ioannou2006tutorial} to hybrid cases and yields a linear parameter estimation equation to estimate the compensator parameters that reject the disturbance without the knowledge of the disturbance state and characteristics. 

    \item The proposed hybrid adaptive controller is integrated into a hierarchical control framework for bipedal locomotion in non-inertial environments. The whole framework is validated in a high-fidelity MuJoCo simulator and a real-world experiment using a Digit robot walking on a platform with unknown and multi-directional motion. We show that the center-of-mass (CoM) tracking error is reduced by an average of $60\%$ in simulations and up to $47\%$ in hardware experiments, confirming the disturbance rejection capability of the proposed controller under complex and real-world uncertainties. 
\end{enumerate}


This paper is organized as follows.
We first focus on the abstracted disturbance rejection problem for hybrid linear systems; Section~\ref{sec: prelim and notation}--\ref{sec:problem formulation} presents preliminaries and problem formulation and Sections~\ref{sec: ideal partial state feedback}--\ref{sec: ada output feedback} propose an adaptive output feedback controller to resolve the abstracted problem.
Then, Section~\ref{sec:case study}--\ref{sec:experiment} applies the resulting adaptive controller to a bipedal robot walking on a moving platform and validate it in simulations and hardware experiments.
Section~\ref{sec:discussion} provides discussions and Section~\ref{sec:conclusion} concludes the work.

\section{Preliminaries and Notation}
\label{sec: prelim and notation}
This section introduces the preliminaries to describe the hybrid error system of a walking robot (see Section~\ref{sec:case study}) and the notation needed when deriving the adaptive controller. 
A linear hybrid dynamical system with the state $\boldsymbol{x}(t)$, continuous phase control input $\boldsymbol{u}_c(t)$, discrete phase control input $\boldsymbol{u}_d(t)$, and measured output $\boldsymbol{y}(t)$ may be described by
\begin{align}
\begin{split} \label{eq: Ltv Hybrid System}
    \dot{\boldsymbol{x}}(t) =& \boldsymbol{A}_c(t)\boldsymbol{x}(t) + \boldsymbol{B}_c(t)\boldsymbol{u}_c(t), \quad t \ne t_k
    \\
    \boldsymbol{x}(t_k^+) =& \boldsymbol{A}_d(t_k)\boldsymbol{x}(t_k^-) + \boldsymbol{B}_d(t_k)\boldsymbol{u}_d(t_k^-), \quad t = t_k
    \\
    \boldsymbol{y}(t) =& \boldsymbol{C}\boldsymbol{x}(t)
\end{split}
\end{align}
\normalsize
where $t_k$ are strictly increasing discrete event times indexed by $k \in\mathbb{N}= \{1,2,3, ...\}$.
The superscripts ${}^{+}$ and ${}^{-}$ attached to the time variable indicate the instant after and before the discrete event time $t_k$.
For simplicity, the conditions $t \ne t_k$ and $t = t_k$ will be omitted in the rest of the paper unless necessary.

A linear hybrid dynamical system is said to be stable if there exist symmetric positive definite matrices $\boldsymbol{P}(t_k)$ and $\boldsymbol{Q}(t_k)$ satisfying the discrete-time Lyapunov equation
$\left(\boldsymbol{A}_d(t_k)\boldsymbol{\Phi}_A(t_k^-,t_{k-1}^+)\right)^T\boldsymbol{P}(t_k)\left(\boldsymbol{A}_d(t_k)\boldsymbol{\Phi}_A(t_k^-,t_{k-1}^+)\right) - \boldsymbol{P}(t_{k-1}) = -\boldsymbol{Q}(t_k)$
where $\boldsymbol{\Phi}_A(\cdot,\cdot)$ is the state-transition matrix corresponding to $\boldsymbol{A}_c(t)$ \cite{hespanha2018linear}.
When the hybrid dynamical system is linear, time-invariant, and the period between discrete events is constant (i.e. $t_k - t_{k-1}=T$ for all $k$), the system is stable when the eigenvalues of the matrix $\boldsymbol{A}_d\exp(\boldsymbol{A}_cT)$ are in the interior of the unit circle \cite{ye1998stability, Heemels2016L2analysis}.

The matrix pair
$( \boldsymbol{F} , \boldsymbol{H} )$ 
is said to be in observable canonical form with the characteristic polynomial
$(s^m + f_1 s^{m-1} + f_2 s^{m-2} + \cdots + f_{m-1} s + f_m)^n$
when they are described by
\begin{align}
\begin{split}
    \boldsymbol{F}
    =&
    \begin{bmatrix}
        -f_{1} \boldsymbol{I} & \boldsymbol{I} & \boldsymbol{0} & \cdots & \boldsymbol{0} \\
        -f_{2} \boldsymbol{I} & \boldsymbol{0}  & \boldsymbol{I} & \cdots  & \boldsymbol{0} \\
        \vdots & \vdots & \vdots & \ddots & \vdots \\
        -f_{m-1} \boldsymbol{I} & \boldsymbol{0} & \boldsymbol{0} & \cdots & \boldsymbol{I} \\
        -f_m \boldsymbol{I} & \boldsymbol{0}  & \boldsymbol{0} & \cdots  & \boldsymbol{0} \\
    \end{bmatrix}
    \in \mathbb{R}^{mn \times mn}
    \\
    \boldsymbol{H}
    =&
    \begin{bmatrix}
    \boldsymbol{I} & \boldsymbol{0} & \boldsymbol{0} & \cdots  & \boldsymbol{0}
    \end{bmatrix}
    \in \mathbb{R}^{n \times mn}
\end{split} \label{eq: F and H canonical form}
\end{align}
\normalsize
where $\boldsymbol{I}$ is the $n \times n$ identity matrix.
The spectrum of a square matrix $\boldsymbol{A}$ is the set containing the eigenvalues of $\boldsymbol{A}$ and is denoted $\sigma(\boldsymbol{A})$.

\section{Problem formulation}
\label{sec:problem formulation}
We formulate the locomotion control problem in non-inertial environments as a disturbance rejection problem for hybrid linear systems subjected to hybrid sinusoidal disturbances (see Section~\ref{sec:case study}).
A general formulation is provided here to enable generalization to other robotics and control applications.


Consider the linear hybrid dynamical system with the plant state $\boldsymbol{x}(t) \in \mathbb{R}^{n_x}$, disturbance state $\boldsymbol{w}(t) \in \mathbb{R}^{n_w}$, continuous phase control input $\boldsymbol{u}_c(t)\in\mathbb{R}^{n_c}$, discrete phase control input $\boldsymbol{u}_d(t)\in\mathbb{R}^{n_d}$, and measured output $\boldsymbol{y} \in \mathbb{R}^{n_y}$ described by

\begin{align}
\begin{split} \label{eq: prob}
\begin{bmatrix}
\dot{\boldsymbol{x}}(t) \\ \dot{\boldsymbol{w}}(t)
\end{bmatrix}
&=
\begin{bmatrix}
\boldsymbol{A}_c & \boldsymbol{P}_c \\
\boldsymbol{0} & \boldsymbol{S}_c
\end{bmatrix}
\begin{bmatrix}
\boldsymbol{x}(t) \\ \boldsymbol{w}(t)
\end{bmatrix}
+
\begin{bmatrix}
\boldsymbol{B}_c \\ \boldsymbol{0}
\end{bmatrix} 
\boldsymbol{u}_c(t)
\\
\begin{bmatrix}
\boldsymbol{x}(t_k^+) \\ \boldsymbol{w}(t_k^+)
\end{bmatrix}
&=
\begin{bmatrix}
\boldsymbol{A}_d & \boldsymbol{P}_d \\
\boldsymbol{0} & \boldsymbol{S}_d
\end{bmatrix}
\begin{bmatrix}
\boldsymbol{x}(t_k^-) \\ \boldsymbol{w}(t_k^-)
\end{bmatrix}
+
\begin{bmatrix}
\boldsymbol{B}_d \\ \boldsymbol{0}
\end{bmatrix} 
\boldsymbol{u}_d(t_k^-)
\\
\boldsymbol{y} &= 
\begin{bmatrix}
\boldsymbol{C} & \boldsymbol{Q}
\end{bmatrix}
\begin{bmatrix}
\boldsymbol{x}(t) \\
\boldsymbol{w}(t)
\end{bmatrix}
\end{split}
\end{align}
\normalsize
Assume that the discrete events occur with a constant duration $T \in \mathbb{R}$, i.e., 
$t_{k+1}-t_k=T$ for all $k\in \mathbb{N}$.
Assume the plant described by the parameters $[ \boldsymbol{A}_c , \boldsymbol{A}_d, \boldsymbol{B}_c, \boldsymbol{B}_d, \boldsymbol{C}]$ is stable~\cite{hespanha2018linear, ye1998stability, Heemels2016L2analysis}, and $\sigma(\boldsymbol{A}_c)\cap\sigma(\boldsymbol{S}_c)=\emptyset$.
Assume the disturbance may be modeled as a hybrid sinusoidal disturbance, i.e., a linear combination of $n_\omega > 0$ sinusoidal components with phase jumps across the discrete events.
This implies the disturbance state matrices may be described by $\boldsymbol{S}_c = \mathrm{diag}\left(
    0,
    \omega_1 \boldsymbol{R}(90^\circ)
    ,
    ...
    ,
    \omega_{n_\omega} \boldsymbol{R}(90^\circ)
    \right)$ and
    $\boldsymbol{S}_d = \mathrm{diag}\left(
    1,
    \boldsymbol{R}(\varphi_1)
    ,
    ...
    ,
    \boldsymbol{R}(\varphi_{n_\omega})
    \right)$ with initial condition $\boldsymbol{w}(0) = [1,1,0,...,1,0]^T$
where $\boldsymbol{R}(\varphi)
    =
    \begin{bsmallmatrix}
        \cos{\varphi}  & -\sin{\varphi} \\
        \sin{\varphi} & \cos{\varphi}
    \end{bsmallmatrix}$ 
, $\omega_{i} > 0$ for $i=1, 2, ..., n_\omega$ are the distinct disturbance frequencies, and $\varphi_i \in \mathbb{R}, i=1,2, ..., n_\omega$ are the phase jumps across the discrete events.
The amplitude and phase of the disturbance are determined by the continuous disturbance matrix $\boldsymbol{P}_c \in \mathbb{R}^{n_x \times n_w}$,
its discrete counterpart $\boldsymbol{P}_d \in \mathbb{R}^{n_x \times n_w}$,
and the disturbance output matrix $\boldsymbol{Q} \in \mathbb{R}^{n_y \times n_w}$.

The control objective is to design a controller which generates the control inputs $\left(\boldsymbol{u}_c(t),\boldsymbol{u}_d(t_k^-)\right)$ such that the measured output $\boldsymbol{y} \rightarrow \boldsymbol{0}$ as $t \rightarrow \infty$.

\section{Solvability of disturbance rejection for hybrid linear systems}
\label{sec: ideal partial state feedback}
This section aims to obtain the necessary and sufficient conditions which make the disturbance rejection problem solvable and to characterize the control inputs that reject the disturbance.
Therefore, the analysis below assumes the disturbance state $\boldsymbol{w}(t)$, the plant parameters $(\boldsymbol{A}_c, \boldsymbol{A}_d, \boldsymbol{B}_c, \boldsymbol{B}_d, \boldsymbol{C}$),
and the disturbance parameters $(\boldsymbol{S}_c, \boldsymbol{S}_d, \boldsymbol{P}_c, \boldsymbol{P}_d, \boldsymbol{Q})$ are known.

Since the plant state $\boldsymbol{x}(t)$ is disturbed by the disturbance state $\boldsymbol{w}(t)$ and the undisturbed system is stable and linear, we hypothesize that the disturbance can be rejected by the partial state feedback controller (which will be shown to be true under some conditions)
\begin{align} \label{eq: state feedback control}
    \boldsymbol{u}_c(t)
    &=
    \begin{bmatrix} \boldsymbol{0} & \boldsymbol{\Psi}_c \end{bmatrix}
    \begin{bmatrix} \boldsymbol{x}(t) \\ \boldsymbol{w}(t) \end{bmatrix}
    &
    \boldsymbol{u}_d(t)
    &=
    \begin{bmatrix} \boldsymbol{0} & \boldsymbol{\Psi}_d \end{bmatrix}
    \begin{bmatrix} \boldsymbol{x}(t) \\ \boldsymbol{w}(t) \end{bmatrix}
\end{align}
\normalsize
where $\boldsymbol{\Psi}_c$ and $\boldsymbol{\Psi}_d$ are the partial state feedback gains with compatible dimensions to be designed.
The closed-loop system using the proposed controller is described by
\begin{align}
\begin{split} \label{eq: state feedback closed loop}
    \begin{bmatrix}
        \dot{\boldsymbol{x}}(t) \\ \dot{\boldsymbol{w}}(t)
    \end{bmatrix}
    =&
    \begin{bmatrix}
        \boldsymbol{A}_c & 
        \boldsymbol{P}_c + \boldsymbol{B}_c \boldsymbol{\Psi}_c\\
        \boldsymbol{0} & \boldsymbol{S}_c
    \end{bmatrix}
    \begin{bmatrix}
        \boldsymbol{x}(t) \\ \boldsymbol{w}(t)
    \end{bmatrix}
    \\
    \begin{bmatrix}
        \boldsymbol{x}(t_k^+) \\ \boldsymbol{w}(t_k^+)
    \end{bmatrix}
    =&
    \begin{bmatrix}
        \boldsymbol{A}_d & 
        \boldsymbol{P}_d + \boldsymbol{B}_d \boldsymbol{\Psi}_d \\
        \boldsymbol{0} & \boldsymbol{S}_d
    \end{bmatrix}
    \begin{bmatrix}
        \boldsymbol{x}(t_k^-) \\ \boldsymbol{w}(t_k^-)
    \end{bmatrix}
    \\
    \boldsymbol{y}
    =&
    \begin{bmatrix}
        \boldsymbol{C} & \boldsymbol{Q}
    \end{bmatrix}
    \begin{bmatrix}
        \boldsymbol{x}(t) \\ \boldsymbol{w}(t)
    \end{bmatrix}
\end{split}
\end{align}
\normalsize
The closed-loop system is autonomous and its steady-state trajectory is characterized in the following lemma.
\begin{lemma}[Steady-state of a Closed-Loop Hybrid Linear System]
\label{lem: ss of hybrid sys.}
    For the closed-loop system \eqref{eq: state feedback closed loop}, the steady-state trajectory of $\boldsymbol{x}(t)$ and $\boldsymbol{y}(t)$ during the $k^{\text{th}}$ continuous phase, i.e., for $t \in [t_k^+, t_{k+1}^-]$, is
    \begin{align} \label{eq: ss of hybrid sys signals}
    \begin{split}
        \boldsymbol{x}(t)
        &=
        \exp{(\boldsymbol{A}_c(t - t_k))} \boldsymbol{\Pi}_d \boldsymbol{w}(t_k^+) 
        +
        \boldsymbol{\Pi}_c \boldsymbol{w}(t)
        \\
        \boldsymbol{y}(t) 
        &=
        \boldsymbol{C} \exp{(\boldsymbol{A}_c(t - t_k))} \boldsymbol{\Pi}_d  \boldsymbol{w}(t_k^+) 
        +
        \left( \boldsymbol{C} \boldsymbol{\Pi}_c + \boldsymbol{Q} \right)\boldsymbol{w}(t)
    \end{split}
    \end{align}
    where $\boldsymbol{\Pi}_c$ and $\boldsymbol{\Pi}_d$ are the solutions to the Sylvester equations
    \begin{align} \label{eq: Sylvester Equations for ss}
    \begin{split}
        \boldsymbol{\Pi}_c \boldsymbol{S}_c =& \boldsymbol{A}_c \boldsymbol{\Pi}_c + \boldsymbol{B}_c \boldsymbol{\Psi}_c + \boldsymbol{P}_c
        \\
        \boldsymbol{\Pi}_d \boldsymbol{S}_d \exp{(\boldsymbol{S}_c T)} =& \boldsymbol{A}_d \exp{(\boldsymbol{A}_c T)}\boldsymbol{\Pi}_d  + \boldsymbol{P}_\Psi \exp{(\boldsymbol{S}_c T)}
    \end{split}
    \end{align}
    where $\boldsymbol{P}_\Psi = 
        \boldsymbol{A}_d
        \boldsymbol{\Pi}_c
        +
        \boldsymbol{B}_d
        \boldsymbol{\Psi}_d
        +
        \boldsymbol{P}_d
        -
        \boldsymbol{\Pi}_c
        \boldsymbol{S}_d$
\end{lemma}
\textit{Proof: given in the appendix.}

Lemma \ref{lem: ss of hybrid sys.} shows that, when a linear hybrid system is forced by a linear combination of sinusoidal components, the steady-state trajectory of $\boldsymbol{x}$ and $\boldsymbol{y}$ is the sum of exponential transient components caused by the discrete events and a linear combination of sinusoidal elements.
To reject the disturbance, the partial state feedback gains should be chosen such that the amplitudes of the exponential transient and sinusoidal components are zero, i.e., $\boldsymbol{\Pi}_d=\boldsymbol{0}$ and $\boldsymbol{C} \boldsymbol{\Pi}_c + \boldsymbol{Q} = \boldsymbol{0}$.
The requirement that $\boldsymbol{\Pi}_d = \boldsymbol{0}$ may be combined with \eqref{eq: Sylvester Equations for ss} to obtain
$\boldsymbol{P}_\Psi \exp{(\boldsymbol{S}_c T)} = \boldsymbol{0}$
or 
$\boldsymbol{P}_\Psi = \boldsymbol{0}$
because the state transition matrix $\exp{(\boldsymbol{S}_c T)}$ is always invertible \cite{hespanha2018linear}.
Therefore, we obtain the following result.


\begin{theorem}[Hybrid Regulator Equations] \label{The: hybird regulator equation}
    For any $\boldsymbol{P}_c$, $\boldsymbol{P}_d$, and $\boldsymbol{Q}$, there exists a control input~\eqref{eq: state feedback control} such that the measured output $\boldsymbol{y} \to \boldsymbol{0}$ as $t \to \infty$ if and only if there exists a solution 
    $\left( \boldsymbol{\Pi}_c , \boldsymbol{\Psi}_c, \boldsymbol{\Psi}_d\right)$ to the 
    \textit{Hybrid Regulator Equations}:
    \begin{align}
    \begin{split}
        \boldsymbol{\Pi}_c \boldsymbol{S}_c =& \boldsymbol{A}_c \boldsymbol{\Pi}_c
        +
        \boldsymbol{B}_c \boldsymbol{\Psi}_c
        +
        \boldsymbol{P}_c
        \\
        \boldsymbol{C} \boldsymbol{\Pi}_c + \boldsymbol{Q} =& \boldsymbol{0}
        \\
        \boldsymbol{P}_\Psi
        =&
        \boldsymbol{0}
        \end{split} \label{eq: continuous time regulator equations}
    \end{align}
    \normalsize
\end{theorem}

Theorem~\ref{The: hybird regulator equation} has the following implications:
(1) Theorem~\ref{The: hybird regulator equation} holds \textbf{for any} $\boldsymbol{P}_c, \boldsymbol{P}_d$ and $\boldsymbol{Q}$, including unmatched disturbances, as long as the hybrid regulator equations are satisfied. 
(2) The hybrid regulator equations are a generalization of the continuous time regulator equations since the additional condition $\boldsymbol{P}_\Psi = \boldsymbol{0}$ can be satisfied given $\boldsymbol{\Psi}_d=\boldsymbol{0}$ and $\boldsymbol{A}_d = \boldsymbol{S}_d = \boldsymbol{I}$, i.e., the system falls back to the continuous-system case.
(3) The new hybrid requirement $\boldsymbol{P}_\Psi = \boldsymbol{0}$ has a clear interpretation: the discrete phase control input must keep $\boldsymbol{x}$ on the steady-state trajectory where $\boldsymbol{y}$ is zero.
(4) The control input that rejects the disturbance could be characterized as a linear combination of the disturbance state, as described in~\eqref{eq: state feedback control}. This encourages and guides the observer-compensator design as discussed in the next section.



The derived hybrid regulator equations can be viewed as a special case of those in~\cite{carnevale2015hybrid} that requires sufficient control authority in $\boldsymbol{B}_d$, as in bipedal walking. This requirement makes the equations in~\eqref{eq: continuous time regulator equations} easier to solve, since the vectorization operator can be applied directly.

\section{Output feedback controller structure design}
\label{sec: ideal output feedback}
This section extends the observer-compensator design of the YK parameterization to hybrid linear systems and shows that the proposed control structure is stable for all controller parameters and can reject the disturbance with only output feedback. 
These characteristics lay the foundation for the use of the adaptive law introduced in the next section.
In this section, the disturbance state is assumed to be unknown but the plant parameters and the disturbance parameters are known. 


\subsection{Proposed output feedback controller}
The proposed controller is comprised of an observer, a continuous phase compensator and a jump compensator, and a discrete phase compensator.
The observer is designed to output a signal $\boldsymbol{d}(t)$ that encodes the effect of the disturbance regardless of the control inputs.
Using $\boldsymbol{d}(t)$ as input, the continuous phase and jump compensators output the continuous control input $\boldsymbol{u}_c(t)$ while the discrete phase compensator outputs the discrete control input $\boldsymbol{u}_d(t)$ to reject the disturbance injected into the hybrid system through both continuous and discrete channels.

\textbf{Observer:}
The observer with state $\hat{\boldsymbol{x}}(t)$ and output $\hat{\boldsymbol{y}}(t)$ predicts the response of the plant to the control inputs without the effect of the disturbance.
The observer is described by
\begin{align}
\begin{split} \label{eq: ideal obsv}
    \boldsymbol{\dot{\hat{x}}}(t) 
    &=
    \boldsymbol{A}_c
    \boldsymbol{\hat{x}}(t)
    +
    \boldsymbol{B}_c
    \boldsymbol{u}_c(t)
    , \quad
    \\
    \boldsymbol{\hat{x}}(t_k^+)
    &=
    \boldsymbol{A}_d 
    \boldsymbol{\hat{x}}(t_k^-)
    +
    \boldsymbol{B}_d
    \boldsymbol{u}_d(t_k^-)
    \\
    \boldsymbol{\hat{y}}(t)
    &=
    \boldsymbol{C} \boldsymbol{\hat{x}}(t)
    \\
\end{split}
\end{align}
\normalsize
With the actual, measured output $\boldsymbol{y}(t)$, the observer further provides information $\boldsymbol{d}(t)$ about the unknown disturbance by removing the effect of the control input from the measured output.
\begin{equation}
    \boldsymbol{d}(t)
    =
    \boldsymbol{y}(t) - \boldsymbol{\hat{y}}(t)
    \label{eq: def d(t)}
\end{equation}

\textbf{Continuous Phase Compensator and Jump Compensator:}
The continuous phase compensator with state $\boldsymbol{\eta}_c(t)$ and the jump compensator with state $\boldsymbol{\eta}_j(t)$ produce the continuous phase control input $\boldsymbol{u}_c(t)$ from the signal $\boldsymbol{d}(t)$.
The key difficulty of extending the observer-compensator design to hybrid systems is that $\boldsymbol{d}(t)$ contains the time-varying transition effect triggered by every discrete event (as shown later in~\eqref{eq:d(t)}), preventing the continuous phase compensator alone from generating the control input $\boldsymbol{u}_c(t)$ that rejects the disturbance.
Therefore, the key idea is the jump compensator that removes the time-varying transition effect from $\boldsymbol{d}(t)$ and allows the output $\boldsymbol{u}_c(t)$ to reject the disturbance, as shown later in Theorem~\ref{The: ctrl design}.
The two compensators are described by
\begin{align}
\begin{split} \label{eq: cont comp}
    \begin{bmatrix}
        \dot{\boldsymbol{\eta}}_c(t) \\ \dot{\boldsymbol{\eta}}_j(t)
    \end{bmatrix}
    &=
    \begin{bmatrix}
        \boldsymbol{F}_c & \boldsymbol{0} \\
        \boldsymbol{0} & \boldsymbol{0}
    \end{bmatrix}
    \begin{bmatrix}
        \boldsymbol{\eta}_c(t) \\
        \boldsymbol{\eta}_j(t)
    \end{bmatrix}
    +
    \begin{bmatrix}
        \boldsymbol{G}_c \\ \boldsymbol{0}
    \end{bmatrix}
    \boldsymbol{d}(t)
    \\
    \begin{bmatrix}
        \boldsymbol{\eta}_c(t_k^+) \\ \boldsymbol{\eta}_j(t_k^+)
    \end{bmatrix}
    &=
    \begin{bmatrix}
        \boldsymbol{I} & \boldsymbol{H}_j \\
        \boldsymbol{0} & \boldsymbol{F}_j
    \end{bmatrix}
    \begin{bmatrix}
        \boldsymbol{\eta}_c(t_k^-) \\ \boldsymbol{\eta}_j(t_k^-)
    \end{bmatrix}
    +
    \begin{bmatrix}
        \boldsymbol{0} \\ \boldsymbol{G}_j
    \end{bmatrix}
    \boldsymbol{d}(t_k^-)
    \\
    \boldsymbol{u}_c(t)
    &=
    \begin{bmatrix}
        \boldsymbol{H}_c & \boldsymbol{0}
    \end{bmatrix}
    \begin{bmatrix}
        \boldsymbol{\eta}_c(t) \\
        \boldsymbol{\eta}_j(t)
    \end{bmatrix}
\end{split}
\end{align}
\normalsize
where the pairs $(\boldsymbol{F}_c,\boldsymbol{H}_c)$ and $(\boldsymbol{F}_j,\boldsymbol{H}_j)$ are in observable canonical form with the characteristic polynomials
$(s^{m_c} + f_{c1}s^{m_c-1} + f_{c2}s^{m_c-2} + \cdots + f_{m_c-1}s + f_{m_c})^{n_c}$
and
$(s^{m_j} + f_{j1}s^{m_j-1} + f_{j2}s^{m_j-2} + \cdots + f_{m_j-1}s + f_{m_j})^{n_c m_c}$ respectively.
The coefficients are chosen by the control designer such that (1) $\boldsymbol{F}_c$ and $\boldsymbol{F}_j$ are stable and (2) $\sigma(\boldsymbol{A}_c) \cap \sigma(\boldsymbol{F}_c) = \emptyset$.
The orders $m_c$ and $m_j$ and the input matrices
$\boldsymbol{G}_c \in \mathbb{R}^{n_c m_c \times n_y}$
and
$\boldsymbol{G}_j \in \mathbb{R}^{n_c m_c m_j \times n_y}$
are the parameters of the compensators to be determined.

\textbf{Discrete Phase Compensator:}
The discrete phase compensator produces the discrete phase control input $\boldsymbol{u}_d(t)$ from the signal $\boldsymbol{d}(t)$ to reject the disturbance and is described by
\begin{align}
\begin{split}
    \dot{\boldsymbol{\eta}}_d(t) 
    &=
    \boldsymbol{0}
    \\
    \boldsymbol{\eta}_d(t_k ^+)
    &=
    \boldsymbol{F}_d
    \boldsymbol{\eta}_d(t_k^-)
    +
    \boldsymbol{G}_d 
    \boldsymbol{d}(t_k^-)
    \\   
    \boldsymbol{u}_d(t)
    &=
    \boldsymbol{H}_d
    \boldsymbol{\eta}_d(t)
\end{split}
\label{eq: disc comp}
\end{align}
\normalsize
where the pair $(\boldsymbol{F}_d,\boldsymbol{H}_d)$ is in observable canonical form with the characteristic polynomial 
$(s^{m_d} + f_{d1}s^{m_d-1} + f_{d2}s^{m_d-2} + \cdots + f_{dm_d-1}s + f_{dm_d})^{n_d}$.
The  coefficients are chosen by the control designer such that $\boldsymbol{F}_d$ is stable.
The order $m_d$ and the input matrix 
$\boldsymbol{G}_d \in \mathbb{R}^{n_d m_d \times n_y}$ 
are the parameters of the compensator to be determined.

\subsection{Controller analysis}
To show that the proposed design can reject the disturbance, the analysis begins by characterizing the behavior of the signal $\boldsymbol{d}(t)$ driving the compensators.
Observe that the system's dynamics are linear, which means the signal $\boldsymbol{d}(t)$ may be viewed as the output of the autonomous sub-system
\begin{align}
\begin{split}
    \begin{bmatrix}
        \dot{\tilde{\boldsymbol{e}}}(t) \\ \dot{\boldsymbol{w}}(t)
    \end{bmatrix}
    &=
    \begin{bmatrix}
        \boldsymbol{A}_c & \boldsymbol{0} \\
        \boldsymbol{0} & \boldsymbol{S}_c
    \end{bmatrix}
    \begin{bmatrix}
        \tilde{\boldsymbol{e}}(t) \\ 
        \boldsymbol{w}(t)
    \end{bmatrix}
    \\
    \begin{bmatrix}
        \tilde{\boldsymbol{e}}(t_k^+) \\ \boldsymbol{w}(t_k^+)
    \end{bmatrix}
    &=
    \begin{bmatrix}
        \boldsymbol{A}_d & \boldsymbol{P}_0 \\
        \boldsymbol{0} & \boldsymbol{S}_d
    \end{bmatrix}
    \begin{bmatrix}
        \tilde{\boldsymbol{e}}(t_k^-) \\ \boldsymbol{w}(t_k^-)
    \end{bmatrix}
    \\
    \boldsymbol{d}(t)
    &=
    \begin{bmatrix}
        \boldsymbol{C} & \boldsymbol{Q}_0
    \end{bmatrix}
    \begin{bmatrix}
        \tilde{\boldsymbol{e}}(t) \\ 
        \boldsymbol{w}(t)
    \end{bmatrix}
\end{split} \label{eq: d generator}
\end{align}
\normalsize
where $\boldsymbol{e}(t) = \boldsymbol{x}(t) - \boldsymbol{\hat{x}}(t)$ is the observer error and $\tilde{\boldsymbol{e}}(t) = \boldsymbol{e}(t) - \boldsymbol{E}_c \boldsymbol{w}(t)$ is the observer error transient.
The matrices $\boldsymbol{P}_0$ and $\boldsymbol{Q}_0$ are given by $\boldsymbol{P}_0 = \boldsymbol{P}_d + \boldsymbol{A}_d\boldsymbol{E}_c - \boldsymbol{E}_c\boldsymbol{S}_d$ and $\boldsymbol{Q}_0 = \boldsymbol{C}\boldsymbol{E}_c + \boldsymbol{Q}$
and the matrix $\boldsymbol{E}_c$ is the solution to the Sylvester equation $\boldsymbol{E}_c \boldsymbol{S}_c = \boldsymbol{A}_c \boldsymbol{E}_c + \boldsymbol{P}_c$.
Lemma~\ref{lem: ss of hybrid sys.} can be used to show that, for $t\in[t_k^+,t_{k+1}^-]$, the sub-system has the steady-state trajectories
\begin{align}
\begin{split}
    \tilde{\boldsymbol{e}}(t)
    =&
    \exp{(\boldsymbol{A}_c(t-t_k))}\boldsymbol{\Pi}_e\boldsymbol{w}(t_k^+)
    \\
    \boldsymbol{w}(t)
    =&
    \exp{(\boldsymbol{S}_c(t-t_k))}\boldsymbol{w}(t_k^+)
    \\
    \boldsymbol{d}(t)
    =&
    \boldsymbol{C}\exp{(\boldsymbol{A}_c(t-t_k))}\boldsymbol{\Pi}_e\boldsymbol{w}(t_k^+)
    +
    \boldsymbol{Q}_0\boldsymbol{w}(t)
\end{split}
\label{eq:d(t)}
\end{align}
\normalsize
where $\boldsymbol{\Pi}_e$ is the solution to the Sylvester equation
\begin{align} \label{eq: Pie}
    \boldsymbol{\Pi}_e \boldsymbol{S}_d \exp{(\boldsymbol{S}_cT)}
    =
    \boldsymbol{A}_d \exp(\boldsymbol{A}_cT) \boldsymbol{\Pi}_e
    +
    \boldsymbol{P}_0 \exp{(\boldsymbol{S}_cT)}
\end{align}
\normalsize
At the instant before the next discrete event, we have $\boldsymbol{d}(t_{k+1}^-) = \boldsymbol{Q}_e\boldsymbol{w}(t_{k+1}^-)$
where $\boldsymbol{Q}_e
    =
    \boldsymbol{C}\exp(\boldsymbol{A}_cT)\boldsymbol{\Pi}_e
    +
    \boldsymbol{Q}_0 \exp{(\boldsymbol{S}_cT)}$.

Given the characteristic of the signal $\boldsymbol{d}(t)$, we show that the proposed compensator can map the signal $\boldsymbol{d}(t)$ to the control inputs $(\boldsymbol{u}_c(t), \boldsymbol{u}_d(t_k^-))$ described by Theorem \ref{The: hybird regulator equation} under the conditions described in the following Theorem.
\begin{theorem}[Performance and Characteristics of the Proposed Controller]
\label{The: ctrl design}
    Consider the system \eqref{eq: prob} with the controller described by \eqref{eq: ideal obsv}--\eqref{eq: disc comp}.
    For any choice of compensator parameters $(\boldsymbol{G}_c,\boldsymbol{G}_j,\boldsymbol{G}_d)$, all signals in the closed-loop system are uniformly bounded.
    If there exists a solution to the regulator equations described in Theorem \ref{The: hybird regulator equation}, the compensator orders are chosen such that 
    \begin{align}
        m_d \ge& \text{dim}(\boldsymbol{w}), &
        m_c \ge& \text{dim}(\boldsymbol{x}) + \text{dim}(\boldsymbol{w}), &
        m_j \ge& \text{dim}(\boldsymbol{w})
        \label{eq:order condition}
    \end{align}
    \normalsize
    and the pairs 
    $(\boldsymbol{S}_d\exp(\boldsymbol{S}_cT),\boldsymbol{Q}_e)$ 
    and
    $( \text{diag}( \boldsymbol{A}_c,\boldsymbol{S}_c) , \begin{bmatrix} \boldsymbol{C} & \boldsymbol{Q}_0 \end{bmatrix} )$
    are observable, then there exist compensator parameters 
    $(\boldsymbol{G}_c, \boldsymbol{G}_j,\boldsymbol{G}_d)$
    such that
    \begin{align*}
    \begin{split}
        \lim_{t\to\infty}\boldsymbol{u}_c(t) =& \boldsymbol{\Psi}_c\boldsymbol{w}(t)
        \\
        \lim_{k\to\infty}\boldsymbol{u}_d(t_k^-) =& \boldsymbol{\Psi}_d\boldsymbol{w}(t_k^-)
        \\
        \lim_{t\to\infty}\boldsymbol{y}(t) =& \boldsymbol{0}
    \end{split}
    \end{align*}
    \normalsize
\end{theorem}
\textit{Proof: in the appendix.}\newline

Theorem~\ref{The: ctrl design} implies the following:
(1) With the proposed controller described by \eqref{eq: ideal obsv}--\eqref{eq: disc comp}, all signals of the closed-loop system are uniformly bounded for any $(\boldsymbol{G}_c,\boldsymbol{G}_j,\boldsymbol{G}_d)$. This allows the adaptive law to alter the compensator parameters without destabilizing the system.
(2) The proposed controller can reject the disturbance, provided that the compensator orders and the parameters are chosen appropriately and a certain observability criterion is met. 
The observability criterion ensures that the measurement $\boldsymbol{d}(t)$ contains sufficient information to reconstruct the disturbance and exponential transients~\cite{marino2002global,liu2009parameter,nikiforov1998adaptive} whereas the order conditions~\eqref{eq:order condition} provide the three compensators with the necessary structural flexibility to remove the transient dynamics and to generate the required control inputs.
(3) This extends the existing observer-compensator design techniques for linear systems to the considered class of hybrid linear systems.
Finally, following~\cite{stewart2026mimo}, the observable-canonical-form assumptions on $(\boldsymbol{F}_c, \boldsymbol{H}_c)$, $(\boldsymbol{F}_j, \boldsymbol{H}_j)$, and $(\boldsymbol{F}_d, \boldsymbol{H}_d)$ can be relaxed, allowing the control designer to select any observable pair for improved numerical stability and computational efficiency.

\section{Adaptive Output Feedback with Unknown Disturbance Parameters}
\label{sec: ada output feedback}
Based on the output feedback controller proposed in the previous section, this section extends it to an adaptive version by constructing the linear parameter estimation equation using the \emph{Hybrid Swapping Lemma} and by adopting a robust adaptive law to estimate the parameters.
Unlike the previous sections, this section removes the assumption of known disturbance states and parameters, making the approach applicable to practical scenarios.

\subsection{Observer-compensator subsystem}
\label{sec: observer-compensator subsystem}
This subsection constructs the observer-compensator subsystem for the adaptive controller and presents the preliminaries for building the linear parameter estimation equation.
The major block of the adaptive output feedback controller remains the same as the combination of an observer and three compensators described by \eqref{eq: ideal obsv}--\eqref{eq: disc comp} except the following:
(1) the fixed compensator parameters $(\boldsymbol{G}_c, \boldsymbol{G}_j,\boldsymbol{G}_d)$ are replaced with time-varying estimates $(\hat{\boldsymbol{G}}_c(t), \hat{\boldsymbol{G}}_j(t),\hat{\boldsymbol{G}}_d(t))$ that are updated by an adaptive law and remain continuous across discrete events,
(2) the compensator orders $m_c$, $m_j$, and $m_d$ are chosen such that the conditions of Theorem~\ref{The: ctrl design} are satisfied given the knowledge of the bounds on the number of disturbance frequencies $\bar{n}_{\omega} \ge \dim(\omega)/2+1$.

To construct the linear parameter estimation equation, we introduce the notion of the observer-compensator subsystem, which is obtained by stacking the dynamics of~\eqref{eq: ideal obsv}--\eqref{eq: disc comp}. 
\begin{align}
\begin{split}
    \dot{\boldsymbol{z}}(t) =& \bar{\boldsymbol{A}}_c \boldsymbol{z}(t) + \bar{\boldsymbol{G}}_c(t)\boldsymbol{d}(t)
    \\
    \boldsymbol{z}(t_k^+) =& \bar{\boldsymbol{A}}_d \boldsymbol{z}(t_k^-) + \bar{\boldsymbol{G}}_d(t_k^-)\boldsymbol{d}(t_k^-)
    \\
    \hat{\boldsymbol{y}}(t) =& \bar{\boldsymbol{C}}\boldsymbol{z}(t)
    \label{eq: controller subsystem1}
\end{split}
\end{align}
where $\boldsymbol{z}(t) = [
        \hat{\boldsymbol{x}}^T(t) \quad \boldsymbol{\eta}_c^T(t) \quad \boldsymbol{\eta}_j^T(t) \quad \boldsymbol{\eta}_d^T(t)
    ]^T$ is the state of the subsystem 
, $\bar{\boldsymbol{A}}_c = \begin{bsmallmatrix} \boldsymbol{A}_c & \boldsymbol{B}_c\boldsymbol{H}_c & \boldsymbol{0} & \boldsymbol{0} \\ \boldsymbol{0} & \boldsymbol{F}_c & \boldsymbol{0} & \boldsymbol{0} \\ \boldsymbol{0} & \boldsymbol{0} & \boldsymbol{0} & \boldsymbol{0} \\ \boldsymbol{0} & \boldsymbol{0} & \boldsymbol{0} & \boldsymbol{0} \end{bsmallmatrix}$,
$\bar{\boldsymbol{G}}_c(t) = \begin{bsmallmatrix} \boldsymbol{0} \\ \hat{\boldsymbol{G}}_c(t) \\ \boldsymbol{0} \\ \boldsymbol{0} \end{bsmallmatrix}$,
$\bar{\boldsymbol{A}}_d = \begin{bsmallmatrix} \boldsymbol{A}_d & \boldsymbol{0} & \boldsymbol{0} & \boldsymbol{B}_d\boldsymbol{H}_d \\ \boldsymbol{0} & \boldsymbol{I} & \boldsymbol{H}_j & \boldsymbol{0} \\ \boldsymbol{0} & \boldsymbol{0} & \boldsymbol{F}_j & \boldsymbol{0} \\ \boldsymbol{0} & \boldsymbol{0} & \boldsymbol{0} & \boldsymbol{F}_d \end{bsmallmatrix}$,
$\bar{\boldsymbol{G}}_d(t) = \begin{bsmallmatrix} \boldsymbol{0} \\ \boldsymbol{0} \\ \hat{\boldsymbol{G}}_j(t) \\ \hat{\boldsymbol{G}}_d(t) \end{bsmallmatrix}$, and
$\bar{\boldsymbol{C}} = \begin{bsmallmatrix} \boldsymbol{C} & \boldsymbol{0} & \boldsymbol{0} & \boldsymbol{0} \end{bsmallmatrix}$.

To isolate the effect of the compensator parameters, we further rewrite the system as
\begin{align}
\begin{split}
    \dot{\boldsymbol{z}}(t)
    =&
    \bar{\boldsymbol{A}}_c \boldsymbol{z}(t)
    +
    \bar{\boldsymbol{B}}_c \boldsymbol{\Xi}_c^T(t)\boldsymbol{\hat{\theta}}_c(t)
    \\
    \boldsymbol{z}(t_k^+) =& \bar{\boldsymbol{A}}_d \boldsymbol{z}(t_k^-)
    + 
    \bar{\boldsymbol{B}}_d
    \begin{bmatrix}
        \boldsymbol{\Xi}_j(t_k^-) & \boldsymbol{0} \\
        \boldsymbol{0} & \boldsymbol{\Xi}_d(t_k^-)
    \end{bmatrix}^T
    \begin{bmatrix}
        \hat{\boldsymbol{\theta}}_j(t_k^-) \\ \hat{\boldsymbol{\theta}}_d(t_k^-)
    \end{bmatrix}
    \\
    \hat{\boldsymbol{y}}(t) =& \bar{\boldsymbol{C}}\boldsymbol{z}(t)
    \label{eq: controller subsystem2}
\end{split}
\end{align}
where
$\boldsymbol{\Xi}_c(t) = \boldsymbol{I}_{n_c m_c} \otimes \boldsymbol{d}^T(t)$,
$\boldsymbol{\Xi}_j(t) = \boldsymbol{I}_{n_c m_c m_j} \otimes \boldsymbol{d}^T(t)$,
$\boldsymbol{\Xi}_d(t) = \boldsymbol{I}_{n_d m_d} \otimes \boldsymbol{d}^T(t)$,
contain and rearrange the original input $\boldsymbol{d}(t)$ to~\eqref{eq: controller subsystem1}, where 
$\otimes$ is Kronecker product.
$\hat{\boldsymbol{\theta}}_c(t) = \text{vec}( \hat{\boldsymbol{G}}_c^T(t) )$,
$\hat{\boldsymbol{\theta}}_j(t) = \text{vec}( \hat{\boldsymbol{G}}_j^T(t) )$, and
$\hat{\boldsymbol{\theta}}_d(t) = \text{vec}( \hat{\boldsymbol{G}}_d^T(t) )$ are the vectorized compensator parameters, $\bar{\boldsymbol{B}}_c =
    \begin{bsmallmatrix}
        \boldsymbol{0} \\ \boldsymbol{I} \\ \boldsymbol{0} \\ \boldsymbol{0}
    \end{bsmallmatrix}$, and $\bar{\boldsymbol{B}}_d =
    \begin{bsmallmatrix}
        \boldsymbol{0} & \boldsymbol{0} \\
        \boldsymbol{0} & \boldsymbol{0} \\
        \boldsymbol{I} & \boldsymbol{0} \\
        \boldsymbol{0} & \boldsymbol{I}
    \end{bsmallmatrix}$
are the new input matrices. 
The rewritten form~\eqref{eq: controller subsystem2} treats the vectorized compensator parameters as an input to generate the output $\hat{\boldsymbol{y}}(t)$, facilitating the construction of a linear parameter estimation equation as discussed next.

\vspace{-0.1in}
\subsection{Linear parameter estimation equation}
A linear parameter estimation equation defines the error that an adaptive law aims to minimize and enables online parameter adaptation to reject the disturbance without knowing the disturbance parameters.
In this subsection, the linear parameter estimation equation is obtained by the following:
(1) By deriving the \emph{Hybrid Swapping Lemma}, $\hat{\boldsymbol{y}}(t)$ is shown to be equal to the product of the vectorized compensator parameters $\hat{\boldsymbol{\theta}}(t)$ and the known regressor matrices $\boldsymbol{\phi}(t)$ in the steady-state. 
(2) Based on the existence of parameters that reject the disturbance (Theorem~\ref{The: ctrl design}) and the definition of $\boldsymbol{d}(t)$ in~\eqref{eq: def d(t)}, a linear parameter estimation equation $\boldsymbol{d}(t)=\boldsymbol{\phi}^T(t) \hat{\boldsymbol{\theta}}(t)$ is obtained.
We start with the \emph{Hybrid Swapping Lemma}.

\begin{lemma}[Hybrid Swapping Lemma] \label{lem: hybrid swapping}
Consider the hybrid linear system~\eqref{eq: controller subsystem2}.
The trajectory on the interval $[t_k^+,t_{k+1}^-]$ is described by
\begin{align}
    \hat{\boldsymbol{y}}(t)
    =&
    \boldsymbol{\varepsilon}(t)
    +
    \boldsymbol{\phi}_c^T(t)\hat{\boldsymbol{\theta}}_c(t)
    +
    \begin{bmatrix}
        \boldsymbol{\phi}_j^T(t)& \boldsymbol{\phi}_d^T(t)
    \end{bmatrix}
    \begin{bmatrix}
        \hat{\boldsymbol{\theta}}_j(t_k^-) \\
        \hat{\boldsymbol{\theta}}_d(t_k^-)
    \end{bmatrix}
    \label{eq: estimated output}
\end{align}
where $\boldsymbol{\varepsilon}(t)$ contains the effect of the initial condition and the swapping terms, which become zero when the compensator parameters stop changing and $t \rightarrow \infty$. 
The exact definition of $\boldsymbol{\varepsilon}(t)$ is provided in the appendix for simplicity. With $\boldsymbol{X}(t_0^+)=\boldsymbol{Y}(t_0^+)=\boldsymbol{0}$, $\boldsymbol{\phi}_c^T(t)$ is the output of 
\begin{align}
\begin{split} \label{eq: continuous regressor matrix generator}
    \dot{\boldsymbol{X}}(t) =& \bar{\boldsymbol{A}}_c\boldsymbol{X}(t) + \boldsymbol{\bar{B}}_c \boldsymbol{\Xi}_c^T(t)
    \\
    \boldsymbol{X}(t_k^+) =& \bar{\boldsymbol{A}}_d\boldsymbol{X}(t_k^-)
    \\
    \boldsymbol{\phi}^T_c(t) =& \bar{\boldsymbol{C}}\boldsymbol{X}(t)
\end{split}
\end{align}
and $[\boldsymbol{\phi}_j^T(t) \quad \boldsymbol{\phi}_d^T(t)]$ is the output of 
\begin{align}
\begin{split} \label{eq: discrete regressor matrix generator}
    \dot{\boldsymbol{Y}}(t) =& \bar{\boldsymbol{A}}_c\boldsymbol{Y}(t)
    \\
    \boldsymbol{Y}(t_k^+) =& \bar{\boldsymbol{A}}_d\boldsymbol{Y}(t_k^-) + \boldsymbol{\bar{B}}_d \begin{bmatrix}
        \boldsymbol{\Xi}_j(t_k^-) & \boldsymbol{0} \\
        \boldsymbol{0} & \boldsymbol{\Xi}_d(t_k^-)
    \end{bmatrix}^T
    \\
    \begin{bmatrix}
        \boldsymbol{\phi}_j^T(t) & \boldsymbol{\phi}_d^T(t)
    \end{bmatrix}&
    =\bar{\boldsymbol{C}}\boldsymbol{Y}(t)
\end{split}
\end{align}
\end{lemma}
\textit{Proof: in the appendix.} 

Lemma~\ref{lem: hybrid swapping} shows that, when $t\rightarrow\infty$ and the compensator parameters stop changing (i.e., $\boldsymbol{\varepsilon}(t) = \boldsymbol{0}$), the observer's output $\hat{\boldsymbol{y}}(t)$ can be expressed as
\begin{equation}
    \hat{\boldsymbol{y}}(t)=-\boldsymbol{\phi}^T(t) \hat{\boldsymbol{\theta}}(t)
    \label{eq: y_hat}
\end{equation}
where $\boldsymbol{\phi}^T(t) = -[\boldsymbol{\phi}_c^T(t)\quad  \boldsymbol{\phi}_j^T(t) \quad \boldsymbol{\phi}_d^T(t)]$ is the regressor matrix that can be generated from~\eqref{eq: continuous regressor matrix generator} and \eqref{eq: discrete regressor matrix generator} whose inputs and matrices are all known and
$\hat{\boldsymbol{\theta}}^T(t) = [\hat{\boldsymbol{\theta}}_c^T(t) \quad \hat{\boldsymbol{\theta}}_j^T(t) \quad \hat{\boldsymbol{\theta}}_d^T(t)]$ is the combined vector of vectorized compensator parameters.

Next, recall the definition of $\boldsymbol{d}(t)$ in~\eqref{eq: def d(t)}, Lemma~\ref{lem: hybrid swapping}, and Theorem~\ref{The: ctrl design} that ensures the existence of constant compensator parameters, called $\boldsymbol{\theta}^*$, which reject the disturbance, under some conditions. 
We have $\boldsymbol{y}(t) = \boldsymbol{d}(t) - \boldsymbol{\phi}^T(t) \boldsymbol{\theta}^* = \boldsymbol{0}$ in the steady-state. This equation can be rearranged as
\begin{equation}
    \boldsymbol{d}(t)=\boldsymbol{\phi}^T(t) \boldsymbol{\theta}^*
    \label{eq: para estimation eq}
\end{equation}
and shows that the steady-state trajectory of $\boldsymbol{d}(t)$ has a known linear relationship to the unknown parameter vector which rejects the disturbance whereas the regressor matrix is known.
Thus, \eqref{eq: para estimation eq} is called the \emph{Linear parameter estimation equation} that is used for the design of the adaptive law introduced next.

\subsection{Robust adaptive law}
\label{sec: robust adaptive law}
As the linear parameter estimation equation~\eqref{eq: para estimation eq} is accurate under some conditions, i.e., steady-state and constant $\hat{\boldsymbol{\theta}}(t)$, the modified recursive least-squares adaptive law~\cite{ioannou2012robust,Ioannou2006tutorial,ioannou1993ltvSystems,salgado1988efra,han1997efraComment} with normalized estimation error is adopted to provide robustness against the estimation error, measurement noise, and model uncertainty.
Based on~\eqref{eq: para estimation eq}, the normalized estimation error is defined as:
\begin{align}
    \boldsymbol{\epsilon}(t) = \frac{\boldsymbol{d}(t)-\boldsymbol{\phi}^T(t)\hat{\boldsymbol{\theta}}(t)}{1 + \text{trace}\left(\boldsymbol{\phi}^T(t)\boldsymbol{K}(t)\boldsymbol{\phi}(t)\right)} \label{eq: estimation error}
\end{align}
where $\boldsymbol{K}(t)$ is the covariance matrix updated by
\begin{align}
\begin{split} \label{eq: cov update}
    \dot{\boldsymbol{K}}(t)
    =&
    -\alpha\frac{\boldsymbol{K}(t)\boldsymbol{\phi}(t)\boldsymbol{\phi}^T(t)\boldsymbol{K}(t)}{1 + \text{trace}\left(\boldsymbol{\phi}^T(t)\boldsymbol{K}(t)\boldsymbol{\phi}(t)\right)}
    \\
    &+ \gamma \boldsymbol{K}(t) + \beta \boldsymbol{I} - \delta \boldsymbol{K}^2(t)
    \\
    \boldsymbol{K}(t_k^+) =& \boldsymbol{K}(t_k^-)
\end{split}
\end{align}
where $\alpha$ is the learning rate, 
$\gamma$ is the forgetting factor,
$\beta$ and $\delta$ are used to ensure the eigenvalues of the covariance matrix remain bounded from below and above.
The initial condition $\boldsymbol{K}(0) = \boldsymbol{K}^T(0) > \boldsymbol{0}$ is chosen such that $\underline{\sigma}\boldsymbol{I} \le \boldsymbol{K}(0) \le \overline{\sigma}\boldsymbol{I}$ where $\underline{\sigma} = \frac{1}{2\delta}( (\gamma-\alpha) + \sqrt{ (\gamma-\alpha)^2+4\beta\delta})$ and $\overline{\sigma} = \frac{1}{2\delta}( \gamma + \sqrt{\gamma^2+4\beta\delta})$.
The parameter vector is updated by a projection method to protect against parameter drift \cite{ioannou2012robust}
\begin{align}
\begin{split} \label{eq: parameter update}
    \dot{\hat{\boldsymbol{\theta}}}(t)
    =&
    \text{proj}\left( \alpha\boldsymbol{K}(t)\boldsymbol{\phi}(t)\boldsymbol{\epsilon}(t) \right)
    \\
    \hat{\boldsymbol{\theta}}(t_k^+)
    =&
    \hat{\boldsymbol{\theta}}(t_k^-)
\end{split}    
\end{align}
where the initial condition is set to $\boldsymbol{0}$ as we have no prior knowledge of the disturbance and the projection is implemented as $\dot{\hat{\boldsymbol{\theta}}}(t) = \boldsymbol{v}(t)$ if $|\hat{\boldsymbol{\theta}}(t)| < \theta_m$ or $\hat{\boldsymbol{\theta}}(t)\cdot \boldsymbol{v}(t) \leq 0$ and $\dot{\hat{\boldsymbol{\theta}}}(t) = (\boldsymbol{I}-\frac{\boldsymbol{K}(t)\hat{\boldsymbol{\theta}}(t)\hat{\boldsymbol{\theta}}^T(t)}{\hat{\boldsymbol{\theta}}^T(t)\boldsymbol{K}(t)\hat{\boldsymbol{\theta}}(t)}) \boldsymbol{v}(t)$ otherwise, where $\boldsymbol{v}(t) = \alpha\boldsymbol{K}(t)\boldsymbol{\phi}(t)\boldsymbol{\epsilon}(t)$ and $\theta_m$ is the radius of the allowed parameter set.

\subsection{Performance of Hybrid Adaptive Controller}
Given the design of the observer-compensator and the robust adaptive law described in Sections~\ref{sec: observer-compensator subsystem} and \ref{sec: robust adaptive law}, we have the following result.

\begin{theorem}[Performance of Hybrid Adaptive Controller]
\label{The: adaptive ctrl performance}
Consider the system~\eqref{eq: prob} with the controller described by~\eqref{eq: ideal obsv}--\eqref{eq: disc comp}, and the parameters generated according to the proposed adaptive law~\eqref{eq: estimation error} - \eqref{eq: parameter update}.
Assume the conditions of Theorems~\ref{The: hybird regulator equation} and~\ref{The: ctrl design} are satisfied.
Then all signals in the closed-loop system are uniformly bounded and the measured output converges to zero asymptotically, i.e., $\lim_{t\to\infty}\boldsymbol{y}(t) = \boldsymbol{0}$.
\end{theorem}
\textit{Proof: in the appendix.}

The analysis in this section shows that existing robust adaptive laws can be used with the controller structure proposed in Section~\ref{sec: ideal output feedback} to solve the hybrid disturbance rejection problem without knowledge of the disturbance parameters.
Although a comprehensive analysis proving robustness to noise, model uncertainty, and parameter variations in the hybrid case remains the subject of future work, we demonstrate that the proposed controller exhibits robustness against model uncertainty and parameter variations in an important use case, bipedal locomotion, as discussed next.

\section{Case study: Bipedal walking on a platform with multi-directional motion}
\label{sec:case study}

In this section, the proposed adaptive controller is applied to control a bipedal robot walking on a platform with unknown, multi-directional motion.
To achieve reactive behavior given the complex robot dynamics, the proposed walking controller is based on the standard three-level hierarchical control framework: high-level motion planner, mid-level motion generator, and low-level motion tracker~\cite{gao2024ALIPDRS}.
The key novelty of the proposed controller lies in the design of the high-level motion planner that integrates the proposed hybrid adaptive controller to achieve improved CoM tracking.

We first introduce the reduced-order model representing the system of a walking bipedal robot on moving surfaces, which is used to design the high-level motion planner and enable efficient and effective real-time motion planning. The model is shown in Fig.~\ref{fig: control pipeline}(a).
Then, our proposed hierarchical control framework, as shown in Fig.~\ref{fig: control pipeline}(b), is introduced while emphasizing the high-level motion planner where the proposed adaptive controller is used.

\subsection{Reduced-order modeling}
\label{sec: reduced-order modeling}
To ensure reactive real-time control, a computationally efficient and fairly faithful reduced-order model is used to approximate the complete robot dynamics.
Consider a linear inverted pendulum (LIP) model on a moving surface~\cite{3D_LIP, stewart2025adaptive} with state
$\boldsymbol{p}(t) = [x_{sc}(t), \dot{x}_{sc}(t)]^T$, where $x_{sc}(t) \in \mathbb{R}$ is the CoM position relative to the supporting point at time $t$, as shown in Fig.\ref{fig: control pipeline}(a):
\begin{align}
    \dot{\boldsymbol{p}}(t)
    =&
    \begin{bmatrix}
        0 & 1 \\
        \frac{g+\ddot{z}_{ws}(t)}{z_{sc}} & 0
    \end{bmatrix}
    \boldsymbol{p}(t)
    +
    \begin{bmatrix}
        0 \\
        - \ddot{x}_{ws}(t)
    \end{bmatrix}
    +
    \begin{bmatrix}
        0 \\
        -\frac{\tau(t)}{mz_{sc}}
    \end{bmatrix}
    \label{eq: lip model conti}
    \\
    \boldsymbol{p}(t_k^+)
    =&
    \boldsymbol{p}(t_k^-)
    +
    \begin{bmatrix}
        - s(t_k^-) \\
        0
    \end{bmatrix}
    + 
    \begin{bmatrix}
        \delta_s(t_k^-) \\
        \delta_r(t_k^-)
    \end{bmatrix}
    \label{eq: lip model discr}
\end{align}
\normalsize
where
$g$ is the gravitational constant,
$m$ is the mass of the robot,
$z_{sc}$ is the constant CoM height relative to the supporting point,
$\ddot{x}_{ws}(t), \ddot{z}_{ws}(t)$ are the unknown horizontal and vertical surface accelerations in the inertial world frame,
$\tau(t)$ is the ankle torque,
$s(t_k^-)$ is the step location command,
$\delta_{s}(t_k^-)$ is the unknown step location error caused by imperfect swing foot control, and
$\delta_{r}(t_k^-)$ is the unknown velocity jump caused by the touchdown impact.
We assume that $\ddot{x}_{ws}(t), \ddot{z}_{ws}(t), \delta_{s}(t_k^-), \delta_{r}(t_k^-)$ are the sums of sinusoidal components and a constant bias, satisfying the assumptions in equation \eqref{eq: prob}.
We also assume that the continuous phase duration is constant, i.e., $t_{k+1} - t_{k} = T$ for all step indices $k \in \mathbb{N}$ where $t_{k}$ denotes the touchdown timing of the stance foot for the $k^{\text{th}}$ step. 

The high-level motion planner regulates two variables that characterize the $k^{\text{th}}$ step's locomotion behavior.
The averaged walking speed for the $k^{\text{th}}$ step is defined as
\begin{align}
    v_{\text{walk}}(k) = 
    \frac{x_{sc}(t_k+T) - x_{sc}(t_k)}{T}
\end{align}
and the walking offset for the $k^{\text{th}}$ step is defined as
\begin{align}
    o_{\text{walk}}(k) =
    \frac{x_{sc}(t_k+T) + x_{sc}(t_k)}{2}
\end{align}
The walking offset describes the horizontal distance between the supporting foot and the mid-point of the CoM position at the beginning and end of the step, as illustrated in Fig.~\ref{fig: control pipeline}(a). 
By selecting a different desired walking offset, the robot can walk using a period-1 or period-2 orbit \cite{xiong_3-d_2022}, unifying the control for both forward and lateral walking.

\subsection{Proposed hierarchical control framework}
\label{sec: hier controller}

\begin{figure*}
    \centering
    \includegraphics[width=0.95\linewidth]{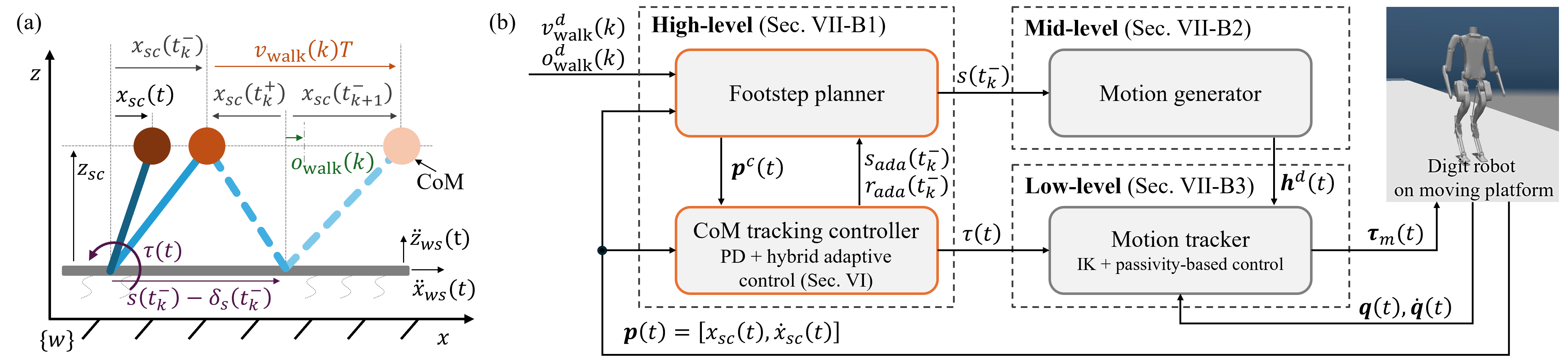}
    \caption{(a) The inverted pendulum model on a moving platform (see Section~\ref{sec: reduced-order modeling}).
    (b) The proposed control framework is a hierarchical control structure (see Section~\ref{sec: hier controller}). 
    }
    \vspace{-0.1in}
    \label{fig: control pipeline}
\end{figure*}

As shown in Fig.~\ref{fig: control pipeline}(b), the proposed hierarchical control framework includes three layers. 
The high-level motion planner takes the user commands as inputs and determines the desired step location and ankle torque.
The mid-level motion generator generates the desired task-space trajectory that respects the reduced-order model assumptions and the step location imposed by the high-level motion planner.
The low-level motion tracker ensures the desired task-space motion plan and the desired ankle torque command are executed.

\subsubsection{High-level motion planner}
The high-level motion planner comprises two sub-modules, i.e., the footstep planner and the CoM tracking controller. 
The footstep planner generates the desired step location and the desired CoM trajectory while the CoM tracking controller determines the ankle torque to track the desired CoM trajectory.

\textbf{Footstep planner:} 
For any $t \in (t_k^+, t_{k+1}^-]$ (i.e., during the $k^{\text{th}}$ step), the control objective of the footstep planner is to minimize the cost function $\mathcal{L}(k+1)$ by generating the step location command $s(t_{k+1}^-)$ and determining the commanded CoM trajectory $\boldsymbol{p}^{c}(t) = [ x_{sc}^{c}(t),  \dot{x}_{sc}^{c}(t)]^T$ for $t \in [t_{k+1}^+ , t_{k+2}^-]$.

Given the user-defined desired averaged walking speed $v_{\text{walk}}^{d}(k)$ and offset $o_{\text{walk}}^{d}(k)$, the cost function $\mathcal{L}(k)$ is defined as
$\mathcal{L}(k) = 
    \boldsymbol{e}_{walk}^T (k)
    \boldsymbol{Q}_{walk}
    \boldsymbol{e}_{walk} (k)$
where $\boldsymbol{e}_{walk}(k) = [v_{\text{walk}}(k), o_{\text{walk}}(k)]^T - [v_{\text{walk}}^{d}(k), o_{\text{walk}}^{d}(k)]^T$ is the output error in walking behavior and $\boldsymbol{Q}_{walk} \in \mathbb{R}^{2\times2}$ is a positive semi-definite matrix determining the weights.

To achieve the objective, the commanded step location is computed by solving the following optimization problem
\begin{align}
    \underset{s(t_{k+1}^-)}{\text{minimize}}
    \quad
    \mathcal{L}(k+1)
    \label{eq:min_L}
\end{align}
The cost function of the next step $\mathcal{L}(k+1)$ is predicted using the LIP model in equation \eqref{eq: lip model conti} with no surface acceleration or ankle torque since we assume no knowledge about the ground motion and the future ankle torque profile.
More details can be found in the appendix. 

To ensure the hybrid error system has a full rank discrete phase input matrix so the conditions of Theorem~\ref{The: hybird regulator equation} are satisfied and to minimize the post-impact ankle torque request, the commanded CoM trajectory $\boldsymbol{p}^c(t)$ of the ${(k+1)}^{\text{th}}$ step is determined in two steps: First, the post-impact commanded CoM state $\boldsymbol{p}^c(t_{k+1}^+)$ is computed as:
\begin{equation}
    \boldsymbol{p}^c(t_{k+1}^+)
    =
    \boldsymbol{p}(t_{k+1}^+)
    +
    [s_{ada}(t_{k+1}^-), r_{ada}(t_{k+1}^-)]^T
    \label{eq: xc update}
\end{equation}
where $\boldsymbol{p}(t_{k+1}^+)$ is the actual post-impact state, $s_{ada}(t_{k+1}^-)$ is the commanded position offset, and $r_{ada}(t_{k+1}^-)$ is the commanded velocity offset. 
$s_{ada}(t_{k+1}^-), r_{ada}(t_{k+1}^-)$ are generated by the adaptive controller as introduced later in~\eqref{eq: jump dyn of e_sc}.
Second, as we assume no knowledge of platform motion and aim to reduce the ankle torque request, the commanded CoM trajectory $\boldsymbol{p}^c(t)$ is computed using the initial condition $\boldsymbol{p}^{c}(t_{k+1}^+)$ from~\eqref{eq: xc update} and the passive LIP dynamics on a static ground:
\begin{align}
    \dot{\boldsymbol{p}}^c(t)
    =
    \begin{bmatrix}
        0 & 1 \\
        \frac{g}{z_{sc}} & 0
    \end{bmatrix}
    \boldsymbol{p}^c(t)
    \label{eq: cmd. conti.}
\end{align}
Although the use of passive LIP on a static ground is a simplification, it remains effective when the platform's acceleration is centered at zero, for example, on an oscillating ship~\cite{chang2026survey}.

\textbf{CoM tracking controller:}
Given the commanded CoM trajectory $\boldsymbol{p}^c(t)$ determined by the footstep planner, the CoM tracking controller generates the ankle torque to track the commanded CoM trajectory.

The CoM tracking error is defined as $\epsilon(t) = x_{sc}(t) - x_{sc}^{c}(t)$, and the ankle torque is designed to ensure the stability assumed in our problem formulation (see Section~\ref{sec:problem formulation}) using proportional and derivative (PD) feedback and gravity compensation: 
\begin{align}
    \tau(t)
    = 
    -mz_{sc}
    \left[
    -\tfrac{g}{z_{sc}} \epsilon(t)
    - k_p \epsilon(t)
    - k_d \dot{\epsilon}(t)
    +
    \tau_{ada}(t)
    \right]
    \label{eq: tau}
\end{align}
\normalsize
where $k_p, k_d$ are the PD gains and $\tau_{ada}(t)$ is the output from the adaptive controller. 

Given the tracking error definition, along with the actual moving-ground LIP dynamics \eqref{eq: lip model conti}, \eqref{eq: lip model discr} and the commanded CoM state dynamics \eqref{eq: cmd. conti.}, \eqref{eq: xc update}, the CoM tracking error dynamics are formulated as
\begin{align}
    \frac{d}{dt}
    \begin{bmatrix}
        \epsilon(t) \\ \dot{\epsilon}(t)
    \end{bmatrix}
    = &
    \boldsymbol{A}_{\Delta}(t)
    \begin{bmatrix}
        \epsilon(t) \\ \dot{\epsilon}(t)
    \end{bmatrix}
    +
    \begin{bmatrix}
        0 \\
        \tilde{d}(t) + \tau_{ada}(t)
    \end{bmatrix}
    \label{eq: conti dyn of e_sc}
    \\
    \begin{bmatrix}
        \epsilon(t_k^+) \\ \dot{\epsilon}(t_k^+)
    \end{bmatrix}
    = &
    \begin{bmatrix}
        0 & 0 \\ 0 & 0
    \end{bmatrix}\begin{bmatrix}
        \epsilon(t_k^-) \\ \dot{\epsilon}(t_k^-)
    \end{bmatrix}
    +
    \begin{bmatrix}
        -1 & 0 \\ 0 & -1
    \end{bmatrix}
    \begin{bmatrix}
        s_{ada}(t_k^-) \\
        r_{ada}(t_k^-)
    \end{bmatrix}
    \label{eq: jump dyn of e_sc}
\end{align}
\normalsize
where the continuous state matrix $\boldsymbol{A}_{\Delta}(t)$ is a stable matrix plus a time-varying parameter uncertainty, i.e., $\boldsymbol{A}_{\Delta}(t) = 
    \begin{bsmallmatrix}
        0 & 1 \\ -k_p & -k_d
    \end{bsmallmatrix}
    +
    \begin{bsmallmatrix}
        0 & 0 \\ \tfrac{\ddot{z}_{ws}(t)}{z_{sc}} & 0
    \end{bsmallmatrix}$, and $\tilde{d}(t) = 
    \tfrac{\ddot{z}_{ws}(t)}{z_{ws}} x_{sc}^{c}(t)
    - \ddot{x}_{ws}(t)$ is the disturbance caused by the ground motion.

By assuming the effect of vertical ground motion is negligible (i.e., $\ddot{z}_{ws}(t)\approx0$), the above hybrid error dynamics match the plant system in equation~\eqref{eq: prob}.
Therefore, the adaptive controller proposed in Section~\ref{sec: ada output feedback} can be applied to drive $\epsilon \to 0$ as $t \to \infty$ by computing the adaptive control signal $(\tau_{ada}(t), s_{ada}(t_k^-), r_{ada}(t_k^-))$.
Note that the above formulation is applicable to both the $x$- and $y$-direction walking and a subscript ${\square}_{x}$ or ${\square}_{y}$ is used to indicate whether the variable is used for the $x$- or $y$-direction.
In addition, the use of the robust adaptive law and the stability provided by the PD feedback in~\eqref{eq: tau} provide robustness to address non-trivial vertical ground motion~\cite{stewart2025adaptive} as we will demonstrate later in Sections~\ref{sec:simulation} and~\ref{sec:experiment}.

\subsubsection{Mid-level motion generator}
The mid-level motion generator~\cite{gao2024ALIPDRS} generates the desired task-space trajectory $\boldsymbol{h}^d(t) \in \mathbb{R}^{18}$ that respects the reduced-order model assumption (e.g., constant CoM height and zero angular momentum) and the step location commands determined by the high-level planner. 
Specifically, to respect the reduced-order model assumption and maintain simplicity, the CoM height $z_{sc}$, trunk orientation $\boldsymbol{\gamma}_{tr}$, upper body joint $\boldsymbol{q}_{up}$, and swing-foot orientation $\boldsymbol{\gamma}_{sw}$ are regulated to be constant.
The desired trajectory for the swing-foot position $\boldsymbol{p}_{sw}(t)$ is directly controlled
to execute the step location command determined by the high-level planner. 
As the key novelty of this work lies in the high-level planner, we refer the readers to~\cite{gao2024ALIPDRS} for the exact formulation.

\subsubsection{Low-level motion tracker}
\label{sec: low-level motion tracker}
The goal of the low-level motion tracker is to generate the motor torque command $\boldsymbol{\tau}_m(t) \in \mathbb{R}^{20}$ to fulfill the motion plan generated by the mid-level generator and the ankle torque command from the high-level planner. 
To ensure robustness against real-world uncertainty, the iterative inverse dynamics and passivity-based control scheme~\cite{gong2022zero, dosunmu2023stair} is adopted. 
Then, the passivity-based controller is used to track the actuator joint trajectory while ensuring the ankle torque applied follows the command $\tau_x(t), \tau_y(t)$ from the high-level planner~\cite{dosunmu2023stair}.
As the major contribution of this work does not involve the low-level motion tracker, interested readers can refer to~\cite{gong2022zero, dosunmu2023stair} for the exact control law.

\subsubsection{Implementation remarks}
\label{sec: implementation remarks}
A few essential implementation details that enable the performance of the system are introduced here.

\textbf{State estimation}:
The full robot state $\boldsymbol{q}(t) \in \mathbb{R}^{36}$ includes a 6-degree-of-freedom floating-base pose and the 30 joint angles on the Digit robot.
In simulation, the MuJoCo ground truth is used to determine the robot state $\boldsymbol{q}(t)$.
In the real-world experiments, the robot state $\boldsymbol{q}(t)$ is provided by the default state estimator of the Digit robot.
The reduced-order model states (e.g., $x_{sc}, y_{sc}, z_{sc}$) are computed using the kinematics model of Digit with the estimated full robot state $\boldsymbol{q}(t)$.

\textbf{Contact detection}:
In practice, the contact timing $t_k$ is determined by the timing when the estimated ground reaction force exceeds a threshold.
This design is made to ensure each stance is started with firm contact between the foot and the ground.
However, this inevitably injects modeling error, i.e., the variation of $t_k$, into the controller, but we show that the controller is robust to such variation.

\textbf{Ankle torque gating mechanism}:
In real-world experiments and simulations, the foot-ground contact may not be firmly established around the impact timing $t_k$. 
Therefore, it is not desirable to apply ankle torque when the foot is about to make or break contact since it might cause foot flipping or oscillation. 
To avoid these situations, a gating mechanism that switches off the ankle torque from $t = t_k-T_{off}^-$ to $t_k+ T_{off}^+$ is implemented, where $T_{off}^-$ and $T_{off}^+$ are determined based on empirical performance.

\textbf{Clipping of adaptive control signal}:
In addition, to avoid the adaptive law injecting aggressive control inputs during the learning process to destabilize the system, a clipping mechanism that clips the adaptive control inputs to ensure $|\tau_{ada}|\le \bar{\tau}_{ada}$, $|s_{ada}|\le\bar{s}_{ada}$, and $|r_{ada}|\le\bar{r}_{ada}$ is added.
While these practical implementations may violate the assumption made in the proposed controller, they are vital to the overall system and the simulation and experimental results show that the controller is robust to these modeling errors.


\section{Simulation Validation}
\label{sec:simulation}

This section evaluates the control framework of Section~\ref{sec:case study} in high-fidelity MuJoCo~\cite{todorov2012mujoco} simulations of a Digit humanoid robot walking on a moving platform whose motion is unknown to the controller. The test cases cover periodic horizontal motion within the disturbance model of Section~\ref{sec:problem formulation} as well as vertical and non-periodic motion outside it, all with one set of controller parameters.

\subsection{Simulation Setup}

\subsubsection{Setup of high-level motion planner}
The parameters of the LIP model are set as: $m = 47$\,kg, $z_{sc} = 0.95$\,m, $T = 0.35$\,s, according to the weight of Digit and its nominal pose.
The desired walking behaviors are chosen as: $o_{\text{walk},x}^d(k) = 0$\,m, $v_{\text{walk},y}^d(k) = 0$\,m/s, and $o_{\text{walk},y}^d(k) = \pm0.13$\,m to have a symmetric walking pattern in the $x$-direction, zero walking speed in the $y$-direction, and a $0.26$\,m step width. 
We test two walking speeds in the forward walking direction as shown in Tab.~\ref{tab: motion profile}.
The weighting matrix $\boldsymbol{Q}_{walk}$ for the footstep planner and the PD gains for the CoM tracking controller $(k_p,k_d)$ are empirically tuned for good performance, for both $x$- and $y$-directions.
The parameters of the ankle torque gating mechanism are set as: $T_{off}^- = 0.053$\,s and $T_{off}^+= 0.007$\,s.

Regarding the adaptive controller, the compensator orders are selected according to Theorem~\ref{The: ctrl design} and to balance the computation load and the performance, and are set as: $m_j=5$, $m_c=10$, and $m_d=5$.
The observable pairs $(\boldsymbol{F}_c,\boldsymbol{H}_c)$, $(\boldsymbol{F}_j,\boldsymbol{H}_j)$, and $(\boldsymbol{F}_d,\boldsymbol{H}_d)$ are designed so that each realizes a low-pass filter bank with a cutoff frequency of approximately 1.9 Hz, which lies above the dominant frequency content of platform motion in typical maritime applications~\cite{chang2026survey}.
The parameter tuning for the adaptive law can be guided by~\cite{ioannou2012robust,Ioannou2006tutorial,salgado1988efra,han1997efraComment}, and the adaptive law parameters are selected as: $\alpha=5\times10^{-1}$, $\beta = 1\times10^{-5}$, $\gamma=1\times10^{-5}$, and $\delta=1\times10^{-9}$.
The parameters of adaptive control input clipping are set as: 
$\bar{\tau}_{ada,x}=3$\,Nm and $\bar{\tau}_{ada,y}=2$\,Nm for first 10 s, 
$\bar{\tau}_{ada,x}=5$\,Nm and $\bar{\tau}_{ada,y}=6$\,Nm for the rest of the time, and 
$\bar{s}_{ada,x}=10$\,mm, 
$\bar{s}_{ada,y}=20$\,mm, 
$\bar{r}_{ada,x}=10$\,mm/s, 
$\bar{r}_{ada,y}=20$\,mm/s for all time. 
One set of parameters is used for all the cases tested below.




\subsubsection{Baseline controller setup}
We compare our proposed control method with two baselines: a) the prior work in~\cite{stewart2025adaptive} and b) the proposed hierarchical control framework without the adaptive controller turned on, called the PD baseline.
For fair comparison, the high-level footstep planner and the ankle torque controller proposed in~\cite{stewart2025adaptive} are integrated with the mid- and low-level controllers used within our proposed hierarchical framework (see Section~\ref{sec: hier controller}).
The performance of the PD baseline is evaluated and recorded before the proposed adaptive controller is turned on.
Both baseline methods are tuned to have their best performance in our simulations.

\subsubsection{Simulation cases}
We test the proposed and baseline controllers under 6 unknown ground motions emulating real-world ship motions~\cite{chang2026survey,han2021data,tannuri2003estimating,zhang2020ship}, as summarized in Table~\ref{tab: motion profile}. 
To test the robustness of our methods, we also include the cases when the surface moves vertically, injecting model uncertainty into the system, see equation~\eqref{eq: conti dyn of e_sc}, and when the surface motion is non-periodic, with slowly-varying amplitude and frequency \cite{stewart2025adaptive, gao2024ALIPDRS}.

\begin{table*}[t]
    \centering
    \caption{Motion profile of dynamic ground and Desired walking speed for simulation study.}
        \vspace{-0.1 in}
    \begin{tabular}{ccccc}
        \toprule
         \textbf{Case} & \multicolumn{3}{c}{\textbf{Ground motion profiles}} &  \textbf{Walking speed}
         \\
            & $x_{ws}(t)$ (m) & $y_{ws}(t)$ (m) & $z_{ws}(t)$ (m) & $v_{\text{walk},x}^{d}$ (m/s) \\ 
        \midrule
        1/2 & 0 & 0 & 0 & 0/0.15
        \\ 
        3/4 & $0.15(1-\cos{(\frac{2\pi}{6} t)})$ & 0 & 0 & 0/0.15
        \\ 
        5/6 & 0 & $0.08(1-\cos{(\frac{2\pi}{5.3} t)})$ & 0 & 0/0.15
        \\ 
        7/8 & $0.15(1-\cos{(\frac{2\pi}{6} t)})$ & 0 & $0.06(1-\cos{(\frac{2\pi}{5} t)})$ & 0/0.15
        \\ 
        9/10 & $0.06(1-\cos{(\frac{2\pi}{4} t)})$ 
        & $0.08(1-\cos{(\frac{2\pi}{5.3} t)})$ 
        & $0.02(1-\cos{(\frac{2\pi}{5} t)})$ & 0/0.15
        \\ 
        11/12 & $0.04(1-\cos{(\frac{2\pi}{5+0.25\sin{(\sqrt{t})}}t)}$ 
        & $\frac{0.05}{1+0.3\sin{(0.7\sqrt{t})}}(1-\cos{(\frac{2\pi}{5}t)})$ 
        & $0.05(1-\cos{(\frac{2\pi}{50}t)})+0.01(1-\cos{(\frac{2\pi}{4.3}t)})$ & 0/0.15
        \\ 
        \bottomrule
    \end{tabular}
    \label{tab: motion profile}
\end{table*}

\subsubsection{Performance metrics}
The CoM tracking error is chosen as a performance metric since it is the output that the hybrid adaptive controller regulates. 
Meanwhile, the total ankle torque request and the errors of the walking speed and walking offset are also reported to provide insights about how the proposed controller improves the performance of robot walking on unknown moving ground.
In simulations, the control framework is operated at 1000\,Hz, and the adaptive controller is turned on when $t=30$\,s and run until $t=80$\,s or longer.
We record performance data at the sample time $t_i=\frac{i-1}{1000}$\,s indexed by $i \in \{1,...,80000\}$.
The performance of the controller without and with adaptive control is evaluated at their steady state, specifically during the evaluation interval $E_{off}=\{ t_i: 0 \leq t_i \leq20\}$ and $E_{on} = \{t_i: 60\leq t_i \leq80\}$, respectively.
The performance metrics as listed in Table~\ref{tab:sim_summary} are defined as follows:
\begin{itemize}
    \item The root-mean-square (RMS) CoM position tracking error $\mathcal{E}_{\epsilon}$ $ = \sum_{t_i \in E} \sqrt{ \frac{1}{M} |\epsilon(t_i)|^2}$, where $M$ is the number of time samples $t_i$ in the evaluation interval $E$.
    \item The RMS total ankle torque $ \mathcal{E}_{{\tau}} = \sum_{t_i \in E} \sqrt{ \frac{1}{M} |\tau(t_i)|^2}$.
    \item The RMS error of walking velocity $ \mathcal{E}_{v}= \sum_{t_i \in E} \sqrt{ \frac{1}{M} |v_{\text{walk}}(t_i) - v_{\text{walk}}^{d}(k(t_i))|^2}$.
    \item The RMS error of walking offset $\mathcal{E}_{o}= \sum_{t_i \in E} \sqrt{ \frac{1}{M} |o_{\text{walk}}(t_i) - o_{\text{walk}}^{d}(k(t_i))|^2}$.
\end{itemize}
Note that these metrics can be applied to both the $x$- and $y$-directions and evaluated at $E_{off}$ and $E_{on}$.

\subsection{Simulation results}

\begin{table*}[!t]
\centering
\caption{Comparison of simulation performance, grouped by ground-motion type. ``PD'' and ``Prop.'' denote the metric evaluated over $E_{\text{off}}$ and $E_{\text{on}}$, respectively. 
$\Delta$ is the mean of the per-case percentage changes within the group; $\Delta_{\min}$ and $\Delta_{\max}$ are the extreme per-case values.
The \textbf{bolded} number highlights the positive improvement.}
\label{tab:sim_summary}
\renewcommand{\arraystretch}{1.15}
\setlength{\tabcolsep}{2.5pt}
\begin{tabular}{@{}cl rrrrr rrrrr rrrrr@{}}
\toprule
& & \multicolumn{5}{c}{Stationary ground (Cases 1--2)}
& \multicolumn{5}{c}{Horizontal, periodic (Cases 3--6)}
& \multicolumn{5}{c}{Vertical and/or non-periodic (Cases 7--12)}\\
\cmidrule(lr){3-7}\cmidrule(lr){8-12}\cmidrule(lr){13-17}
Dir. & Metric
 & PD & Prop. & $\Delta$ (\%)& $\Delta_{\min}$ (\%) & $\Delta_{\max}$ (\%)
 & PD & Prop. & $\Delta$ (\%)& $\Delta_{\min}$ (\%) & $\Delta_{\max}$ (\%)
 & PD & Prop. & $\Delta$ (\%)& $\Delta_{\min}$ (\%) & $\Delta_{\max}$ (\%)\\
\midrule
\multirow{4}{*}{$x$}
 & $\mathcal{E}_{\epsilon}$ (mm)
 & 1.95 & 0.53 & $\mathbf{-73}$ & $\mathbf{-73}$ & $\mathbf{-73}$
 & 2.57 & 0.78 & $\mathbf{-72}$ & $\mathbf{-85}$ & $\mathbf{-61}$
 & 2.97 & 1.18 & $\mathbf{-63}$ & $\mathbf{-85}$ & $\mathbf{-50}$\\
 & $\mathcal{E}_{\tau}$ (N$\cdot$m)
 & 1.95 & 0.89 & $\mathbf{-49}$ & $\mathbf{-64}$ & $\mathbf{-34}$
 & 2.71 & 2.35 & $\mathbf{-20}$ & $\mathbf{-66}$ & $+54$
 & 3.15 & 3.36 & $+5$ & $\mathbf{-49}$ & $+53$\\
 & $\mathcal{E}_{v}$ (mm/s)
 & 26.90 & 35.19 & $+15$ & $\mathbf{-33}$ & $+63$
 & 37.09 & 30.84 & $\mathbf{-23}$ & $\mathbf{-76}$ & $+18$
 & 42.96 & 40.43 & $\mathbf{-5}$ & $\mathbf{-36}$ & $+24$\\
 & $\mathcal{E}_{o}$ (mm)
 & 2.10 & 8.16 & $+310$ & $+224$ & $+397$
 & 5.40 & 9.05 & $+100$ & $+17$ & $+272$
 & 6.59 & 8.21 & $+30$ & $\mathbf{-6}$ & $+71$\\
\addlinespace[2pt]
\multirow{4}{*}{$y$}
 & $\mathcal{E}_{\epsilon}$ (mm)
 & 4.47 & 0.44 & $\mathbf{-90}$ & $\mathbf{-92}$ & $\mathbf{-88}$
 & 4.80 & 1.56 & $\mathbf{-68}$ & $\mathbf{-91}$ & $\mathbf{-39}$
 & 4.98 & 1.41 & $\mathbf{-71}$ & $\mathbf{-90}$ & $\mathbf{-55}$\\
 & $\mathcal{E}_{\tau}$ (N$\cdot$m)
 & 4.39 & 6.82 & $+59$ & $+34$ & $+83$
 & 4.64 & 6.43 & $+39$ & $+29$ & $+44$
 & 4.76 & 5.84 & $+23$ & $+4$ & $+36$\\
 & $\mathcal{E}_{v}$ (mm/s)
 & 17.52 & 14.07 & $\mathbf{-20}$ & $\mathbf{-46}$ & $+7$
 & 23.58 & 19.37 & $\mathbf{-16}$ & $\mathbf{-29}$ & $+0.3$
 & 22.95 & 16.41 & $\mathbf{-26}$ & $\mathbf{-43}$ & $\mathbf{-10}$\\
 & $\mathcal{E}_{o}$ (mm)
 & 31.17 & 28.66 & $\mathbf{-6}$ & $\mathbf{-21}$ & $+9$
 & 33.09 & 18.12 & $\mathbf{-43}$ & $\mathbf{-81}$ & $\mathbf{-3}$
 & 33.24 & 13.67 & $\mathbf{-57}$ & $\mathbf{-77}$ & $\mathbf{-0.2}$\\
\bottomrule
\end{tabular}
\end{table*}

\subsubsection{Comparison on walking stability}
For all cases listed in Table~\ref{tab: motion profile}, the robot walks stably for the full 80\,s trial under both the proposed method and the PD baseline, whereas the robot using the controller of~\cite{stewart2025adaptive} falls within the first 10 steps, even on stationary ground and with gains tuned specifically for that condition. This failure is caused by two structural design choices: First, the step location commands in~\cite{stewart2025adaptive} for all $t>0$ are not updated using the current CoM state, preventing the robot from recovering stability using foot steps. Second, their formulation assumes the actual step location is exactly the same as the commanded step location, the key to transforming the hybrid error system to a continuous system. However, this is an unrealistic assumption for high-fidelity simulation and real-world deployment whose dynamics are more complex than those used in the low-level motion tracker (see Section~\ref{sec: low-level motion tracker}). As the prior work~\cite{stewart2025adaptive} fails to maintain stable walking, we omit the comparison to~\cite{stewart2025adaptive} in the rest of the paper.

\subsubsection{Comparison on CoM tracking error}
Compared to the PD baseline, the proposed adaptive controller reduces the RMS CoM tracking errors $\mathcal{E}_\epsilon$ in both $x$- and $y$-directions by more than $39\%$ and up to $92\%$ across all ground motions. 
This validates the disturbance rejection capability in hybrid linear systems under the hybrid sinusoidal disturbance with unknown parameters.

\subsubsection{Comparison on ankle torque request}
The adaptive controller can reduce the request of control input to save energy given a feasible desired trajectory to track.
The proposed adaptive controller reduces the $x$-direction ankle torque by up to $66\%$ on moving platforms, as it precisely fits the disturbance profile instead of over-activating PD feedback. 
Conversely, the $y$-direction torque increases for all cases, a necessary trade-off for the control law to counteract unmodeled dynamics and modeling error and drive the system toward the passive LIP trajectory.

\subsubsection{Comparison on walking speed and offset tracking}
Integrating the adaptive loop into the CoM controller may also reduce the RMS walking speed and offset errors, thereby enhancing high-level gait execution.
In the $y$-direction, both the walking speed error $\mathcal{E}_{v}$ and the offset error $\mathcal{E}_o$ are improved by the proposed adaptive controller across the three ground motion types. 
In the $x$-direction, however, $\mathcal{E}_v$ is improved in some cases and $\mathcal{E}_o$ increases in almost all cases.

These results reflect the trade-off of the proposed control framework where the adaptive controller adjusts the commanded CoM trajectory~\eqref{eq: xc update}.
This design ensures that the input matrix in~\eqref{eq: jump dyn of e_sc} is full rank, allowing the hybrid adaptive controller to reduce the CoM tracking error, but it also interferes with the minimization of cost function in~\eqref{eq:min_L}. 
As the baseline offset tracking error in the $x$-direction is small (sub-centimeter), the effect of the adaptive control pulls the error to a higher but acceptable value; in the $y$-direction, the same effect reduces the tracking error as the baseline error is larger.

\subsubsection{Robustness to unknown vertical and non-periodic ground motion} 
The proposed control framework remains robust to the unmodeled unknown vertical ground motion and the non-periodic ground motion in all three directions. 
Based on Table~\ref{tab:sim_summary}, the performance improvements, including the reduction in the CoM tracking error in the $x$- and $y$-directions and the improvement in the walking speed and offset tracking errors in the $y$-direction, are preserved in the presence of vertical ground motion and non-periodic ground motion.

\subsubsection{Remarks on hybrid error dynamics and the adaptive control inputs}
Fig.~\ref{fig: case37 x} illustrates the output characteristics under two representative scenarios: Case 3 aligns with the ideal formulation in Section~\ref{sec:problem formulation}, whereas Case 7 injects vertical surface-induced model uncertainty~\eqref{eq: conti dyn of e_sc}.
Prior to adaptation ($t < 30$\,s), the CoM tracking error $\epsilon_{x}$ exhibits a fish-bone pattern with a sinusoidal envelope; this explicitly reflects the hybrid linear dynamics under ground disturbances, where the state periodically resets to zero~\eqref{eq: jump dyn of e_sc}.
Upon activation of the adaptive law at $t = 30$\,s, the controller generates an almost continuous sinusoidal ankle torque $\tau_{ada}$ and discrete updates $(s_{ada}, r_{ada})$, yielding a rapid reduction in $\epsilon_x$ after a 10-second learning transient.
Notably, $\tau_{ada}$ self-synchronizes with the 6-second ground motion period specified in Table~\ref{tab: motion profile}, confirming the controller's capacity to implicitly estimate and reject the disturbance.



\begin{figure*}[t]
    \centering
    \includegraphics{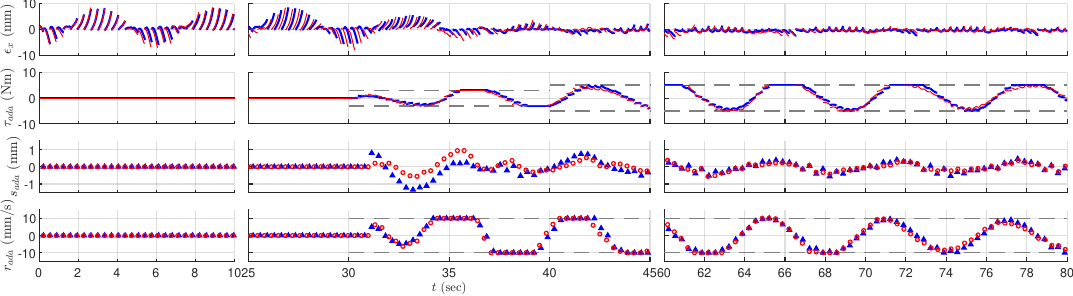}
    \vspace{-0.25in}
    \caption{The CoM tracking error $\epsilon_x$, adaptive ankle torque request $\tau_{ada}$, and discrete control input $s_{ada}, r_{ada}$ in the $x$-direction for case 3 (thicker blue line or triangles) and case 7 (thinner red line or hollow circles) of the simulation validation.
    The adaptive controller is turned on at $t=30$ and the dashed gray lines indicate the limit of control input request.
    }
    \label{fig: case37 x}
\end{figure*}

\section{Real-world experiment}
\label{sec:experiment}
Hardware experiments on the Digit humanoid robot validate the controller's disturbance rejection and CoM tracking capabilities under real-world uncertainties. 
An experiment movie is provided in the supplementary material \url{https://youtu.be/2wP6VvRaFQc}.

\subsection{Experiment setup}

\begin{figure}
    \centering
    \includegraphics[width=0.65\linewidth]{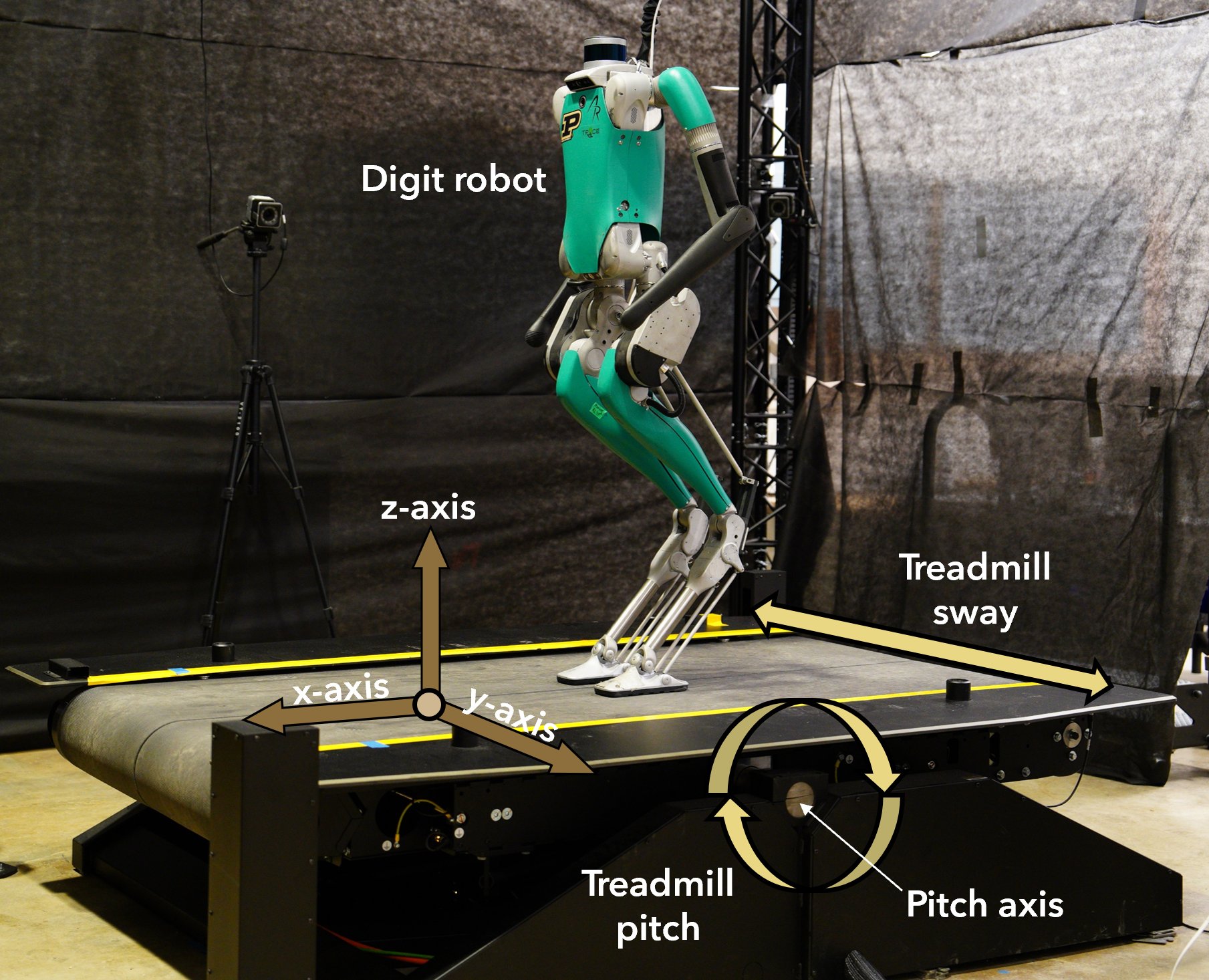}
    \caption{Experiment setup: A humanoid robot walking on a treadmill that can pitch and sway to emulate the ground motion. The motion profile of the treadmill is unknown to the robot controller.}
    \label{fig:exp setup}
    \vspace{-0.1 in}
\end{figure}

\subsubsection{Setup of high-level motion planner}

The setup of the high-level motion planner remains largely unchanged other than the following: 
The $\boldsymbol{Q}_{walk}$ for the footstep planner and the $(k_p,k_d)$ for the CoM tracking controller are re-tuned for hardware experiments, for both $x$- and $y$-directions.
To allow a wider tolerance around the impact timing, the ankle torque gating parameters are adjusted: $T_{off}^{-}= 0.07$\,s and $T_{off}^{+}= 0.07$\,s.
The adaptive parameters remain the same as those used in the simulation validation.
The proposed framework is running on a laptop with an Intel Ultra 9 185H CPU and the whole control framework, including the high-level motion planner, the adaptive law, and the low-level motion tracking, can be stably run at 1k\,Hz.


\subsubsection{Baseline controller setup}
As the prior work~\cite{stewart2025adaptive} is not stable in the simulation, the baseline controller for the experiment only includes the PD baseline, which is the proposed framework without the adaptive controller.


\begin{table*}[!t]
\centering
\caption{Motion profile of the treadmill and the corresponding real-world
experimental performance. ``PD'' and ``Prop.'' denote the
metric for PD baseline and the proposed adaptive framework, respectively, and
$\Delta$ is the corresponding percentage change. 
The bolded numbers highlight the positive improvement.}
\label{tab:exp_summary}
\renewcommand{\arraystretch}{1.15}
\setlength{\tabcolsep}{3pt}
\begin{tabular}{@{}cl rrr rrr rrr rrr@{}}
\toprule
& & \multicolumn{3}{c}{\textbf{Case A}} & \multicolumn{3}{c}{\textbf{Case B}}
& \multicolumn{3}{c}{\textbf{Case C}} & \multicolumn{3}{c}{\textbf{Case D}}\\
& Pitch (deg) & \multicolumn{3}{c}{$4\sin{(\frac{2\pi}{8}t)}$}
& \multicolumn{3}{c}{$4\sin{(\frac{2\pi}{6}t)}$}
& \multicolumn{3}{c}{}
& \multicolumn{3}{c}{$3\sin{(\frac{2\pi}{8}t)}$} \\
& Sway (cm) & \multicolumn{3}{c}{}
& \multicolumn{3}{c}{}
& \multicolumn{3}{c}{$5(1-\cos{(\frac{2\pi}{8}t)})$}
& \multicolumn{3}{c}{$3(1-\cos{(\frac{2\pi}{10}t)})$} \\
\cmidrule(lr){3-5}\cmidrule(lr){6-8}\cmidrule(lr){9-11}\cmidrule(lr){12-14}
Dir. & Metric
 & PD & Prop. & $\Delta$ (\%)
 & PD & Prop. & $\Delta$ (\%)
 & PD & Prop. & $\Delta$ (\%)
 & PD & Prop. & $\Delta$ (\%)\\
\midrule
\multirow{4}{*}{$x$}
 & $\mathcal{E}_{\epsilon}$ (mm)
 & 5.60 & 4.30 & $\mathbf{-23}$ & 7.70 & 4.10 & $\mathbf{-47}$
 & 10.27 & 8.92 & $\mathbf{-13}$ & 7.48 & 4.15 & $\mathbf{-45}$\\
 & $\mathcal{E}_{\tau}$ (N$\cdot$m)
 & 3.39 & 3.98 & $+17$ & 4.54 & 4.79 & $+6$
 & 5.75 & 5.09 & $\mathbf{-11}$ & 4.39 & 3.59 & $\mathbf{-18}$\\
 & $\mathcal{E}_{v}$ (mm/s)
 & 79.50 & 81.10 & $+2$ & 58.10 & 73.60 & $+27$
 & 32.60 & 41.33 & $+27$ & 50.22 & 74.53 & $+48$\\
 & $\mathcal{E}_{o}$ (mm)
 & 20.50 & 19.00 & $\mathbf{-7}$ & 18.30 & 16.70 & $\mathbf{-9}$
 & 11.47 & 11.78 & $+3$ & 18.59 & 16.86 & $\mathbf{-9}$\\
\addlinespace[2pt]
\multirow{4}{*}{$y$}
 & $\mathcal{E}_{\epsilon}$ (mm)
 & 4.90 & 5.70 & $+16$ & 4.90 & 4.80 & $\mathbf{-2}$
 & 5.62 & 5.32 & $\mathbf{-5}$ & 4.61 & 4.49 & $\mathbf{-3}$\\
 & $\mathcal{E}_{\tau}$ (N$\cdot$m)
 & 1.64 & 1.96 & $+20$ & 1.66 & 1.59 & $\mathbf{-4}$
 & 1.96 & 3.18 & $+62$ & 1.63 & 2.56 & $+57$\\
 & $\mathcal{E}_{v}$ (mm/s)
 & 31.00 & 40.50 & $+31$ & 42.10 & 32.60 & $\mathbf{-23}$
 & 51.87 & 48.15 & $\mathbf{-7}$ & 42.60 & 26.72 & $\mathbf{-37}$\\
 & $\mathcal{E}_{o}$ (mm)
 & 57.90 & 51.80 & $\mathbf{-11}$ & 62.00 & 60.20 & $\mathbf{-3}$
 & 68.45 & 46.71 & $\mathbf{-32}$ & 61.17 & 45.32 & $\mathbf{-26}$\\
\bottomrule
\end{tabular}
\end{table*}

\subsubsection{Experiment cases}
The hardware experiment setup is shown in Fig.~\ref{fig:exp setup}.
The Digit robot steps in-place on a moving treadmill configured to pitch and sway, emulating vessel motions derived from empirical sea-state data~\cite{chang2026survey,han2021data,tannuri2003estimating,zhang2020ship}. The motion profile of the treadmill used to test the robot is shown in Table~\ref{tab:exp_summary}.
As the platform's pitching axis is below the platform surface, the rotational motion induces coupled multi-directional base accelerations in both $x$- and $z$-directions, while the sway motion drives lateral ($y$-direction) disturbances. 
For each case, one trial of data with a duration of more than 100\,s is recorded to capture representative behavior of the system.

\subsubsection{Performance metrics}
The same performance metrics used in simulations are adopted for the experiments. The measurement used to compute the metrics is obtained from the proprioceptive sensor on the Digit robot and its default estimator.

\subsection{Experiment results}

\subsubsection{Comparison on CoM tracking error}
The hardware performance results shown in Table~\ref{tab:exp_summary} reveal that the adaptive law reduces the RMS CoM tracking error $\mathcal{E}_\epsilon$ in the $x$-direction by $13\sim47\%$ across all four test cases, confirming effective disturbance rejection under real-world uncertainties. 
In the $y$-direction, $\mathcal{E}_\epsilon$ remains comparable to the baseline, showing slight reductions in Case~B--D and a minor increase in Case~A. 
The limited performance improvement in the $y$-direction is caused by the noisy CoM lateral position and velocity estimation, and the limited ankle torque associated with the smaller support region in the $y$-direction.

\subsubsection{Comparison on ankle torque request}
The ankle torque request in the $x$-direction is reduced by up to $18\%$ in Case D but is increased in Case A--B. 
Similar to the case in the simulation, the increase of the ankle torque request can be explained by the additional torque required to counteract the modeling error and external disturbance.
In addition, Table~\ref{tab:exp_summary} shows that the hardware ankle torque request in the $y$-direction is smaller than that in simulations.
This is caused by the parameter tuning to avoid foot flipping in the real-world experiments caused by Digit's narrow foot pad, explaining the limited performance improvement in the $y$-direction.

\subsubsection{Comparison on walking speed and offset}
Due to the design of the proposed framework, the improvement in RMS CoM tracking error $\mathcal{E}_\epsilon$ may not necessarily improve the walking speed and offset tracking.
In the $y$-direction, the walking speed and offset tracking are all improved except the walking speed in Case A, showing a similar trend observed in the simulation study.
In the $x$-direction, the walking speed errors for all four cases increase while the walking offset errors reduce except for Case C, demonstrating the trade-off inherent to the design of adjusting the commanded trajectory~\eqref{eq: xc update}.

\subsubsection{Remarks on time-series data}
Using Case A as a representative scenario, Fig.~\ref{fig:real exp caseA} shows the trajectory of the CoM tracking error in the $x$-direction and the adaptive control input $\tau_{ada}, s_{ada}, r_{ada}$.
Prior to adaptation ($t < 30$\,s), $\epsilon_x$ displays a periodic envelope superimposed on a $-5$ mm constant bias and a high-frequency, V-shaped perturbation, capturing the combined effect of the treadmill's pitching, unmodeled constant bias, and high-frequency disturbances caused by the low-level controller. 
Upon activating the adaptive law at $t=30$\,s, the continuous ankle torque $\tau_{ada}$ and discrete signals $(s_{ada}, r_{ada})$ are turned on and strictly bounded to prevent aggressive behavior. 
After $t=60$s, the adaptive controller successfully suppresses the primary low-frequency disturbance components by generating periodic ankle torque $\tau_{ada}$ and discrete control signals $s_{ada}, r_{ada}$. 
Given the physical power and bandwidth limitations of the hardware ankle actuators, high-frequency perturbations are deliberately filtered by restricting the compensator parameters $\boldsymbol{F}_c, \boldsymbol{F}_j,$ and $\boldsymbol{F}_d$ to a conservative cut-off frequency rather than attempting aggressive regulation.

\begin{figure*}[t]
    \centering
    \includegraphics{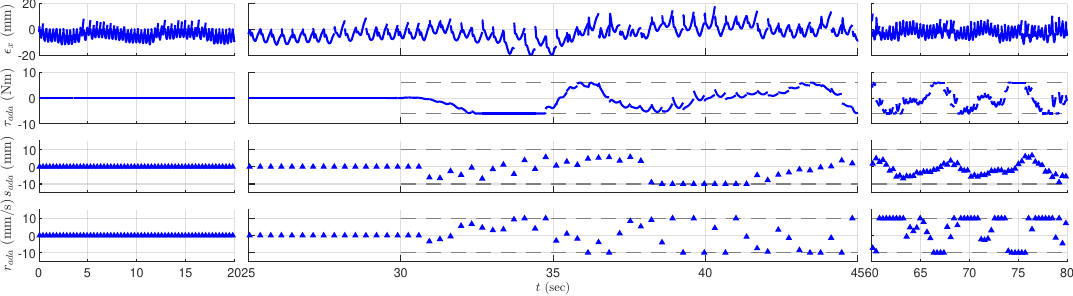}
    \vspace{-0.25in}
    \caption{The CoM tracking error $\epsilon_x$, adaptive ankle torque request $\tau_{ada}$, and discrete control input $s_{ada}, r_{ada}$ in the $x$-direction for case A.
    The adaptive controller is turned on at $t=30$ and the dashed gray lines indicate the limit of control input request.}
    \label{fig:real exp caseA}
\end{figure*}


\section{Discussion}
\label{sec:discussion}
\subsection{Effect of stance period variation}
While the stance period $T$ in both simulations and hardware experiments is not a fixed value but determined by the actual contact forces, the validation results show that the proposed framework is robust to the stance period variation of about $4\%$.
The fixed stance period $T$ assumption used in the problem formulation~\eqref{eq: prob} and the LIP model~\eqref{eq: lip model discr} ensures the step-to-step dynamics can be easily computed, e.g., $\boldsymbol{x}(t_k^+) = \boldsymbol{A}_d \exp(\boldsymbol{A}_{c}T) \boldsymbol{x}(t_{k-1}^+)$ for~\eqref{eq: prob} without control and disturbance inputs.
Variation in $T$ induces errors in the cost function prediction $\mathcal{L}(k+1)$ in~\eqref{eq:min_L} and time-variation in the compensator parameters $(\boldsymbol{G}_c, \boldsymbol{G}_j,\boldsymbol{G}_d)$ in Theorem~\ref{The: ctrl design}, making it difficult for the adaptive law to estimate the noisy and time-varying controller parameters~\cite{stewart2024ltvSystem, stewart2026mimo}.
The stance-period variation also affects the time-varying properties of the solution to Theorem~\ref{The: ctrl design} and the robustness of the proposed controller (Section~\ref{sec: ada output feedback}). A formal analysis of these effects remains open and is left for future work. 

\subsection{Effect of ankle torque gating mechanism}
To ensure stable walking behavior and smooth contact transition in simulations and hardware experiments, the ankle torque is clipped to zero when the foot is about to lift-off or touchdown (see Section~\ref{sec: implementation remarks}).
However, clipping the ankle torque in~\eqref{eq: tau} induces additional unmodeled error in~\eqref{eq: jump dyn of e_sc}, removes the control authority of the adaptive control law intermittently, i.e., making $\boldsymbol{B}_c$ in~\eqref{eq: prob} time-varying, and causes an error when using the systems in~\eqref{eq: continuous regressor matrix generator} and~\eqref{eq: discrete regressor matrix generator} to estimate the regressor matrices in~\eqref{eq: y_hat}.
While the current validation shows positive results with the direct clipping, a formal analysis and proper treatment of input clipping are highly valuable for systems with input constraints, e.g., limited ankle torque for legged systems, and we leave these topics to future work.

\section{Conclusion}
\label{sec:conclusion}

This paper addressed the problem of rejecting unknown hybrid sinusoidal disturbances for a class of hybrid linear systems, motivated by bipedal locomotion in non-inertial environments.
We established necessary and sufficient solvability conditions of the problem via the derived \emph{Hybrid Regulator Equations}. 
To solve the problem, we developed an adaptive output feedback controller by extending the observer-compensator structure of the YK parameterization to hybrid systems and deriving the \emph{Hybrid Swapping Lemma}, which yields a linear parameter estimation equation for estimating the controller parameters that reject the disturbance.
The proposed adaptive controller was integrated into a hierarchical control framework for bipedal locomotion and validated through high-fidelity MuJoCo simulations and real-world experiments on the Digit humanoid robot walking on a moving treadmill with multi-directional unknown motion, reducing center-of-mass tracking error by an average of over 60\% in simulations and up to 47\% in real-world experiments.

\section*{Acknowledgments}
This work has been supported by the National Science Foundation under Grants CMMI-1934280 and CMMI-2046562 and by the Office of Naval Research under Grant N00014-24-1-2028.

\section*{Appendix}
\subsection{Proof of Lemma \ref{lem: ss of hybrid sys.}}
To begin, rewrite \eqref{eq: state feedback closed loop} using the change of variables
$
    \tilde{\boldsymbol{x}}
    =
    \boldsymbol{x} 
    -
    \boldsymbol{\Pi}_c 
    \boldsymbol{w}
$
where $\boldsymbol{\Pi}_c$ is the solution to the Sylvester equation~\eqref{eq: Sylvester Equations for ss}.
The system in the new coordinates is described by
\footnotesize
\begin{align}
\begin{split}
    \begin{bmatrix}
        \dot{\tilde{\boldsymbol{x}}}(t) \\
        \dot{\boldsymbol{w}}(t)
    \end{bmatrix}
    &=
    \begin{bmatrix}
        \boldsymbol{A}_c & 
        \boldsymbol{0} \\
        \boldsymbol{0} & \boldsymbol{S}_c
    \end{bmatrix}
    \begin{bmatrix}
    \tilde{\boldsymbol{x}}(t) \\
    \boldsymbol{w}(t)
    \end{bmatrix}
    \\
    \begin{bmatrix}
        \tilde{\boldsymbol{x}}(t_k^+) \\
        \boldsymbol{w}(t_k^+)
    \end{bmatrix}
    &=
    \begin{bmatrix}
        \boldsymbol{A}_d & \boldsymbol{P}_\Psi \\
        \boldsymbol{0} & \boldsymbol{S}_d
    \end{bmatrix}
    \begin{bmatrix}
        \tilde{\boldsymbol{x}}(t_k^-) \\
        \boldsymbol{w}(t_k^-)
    \end{bmatrix}
    \\
    \boldsymbol{y} 
    &= 
    \begin{bmatrix}
    \boldsymbol{C} & (\boldsymbol{C} \boldsymbol{\Pi}_c +\boldsymbol{Q} )
    \end{bmatrix}
    \begin{bmatrix}
    \tilde{\boldsymbol{x}} \\
    \boldsymbol{w}
    \end{bmatrix}
\end{split} \label{eq: diagonal continuous phase}
\end{align}
\normalsize
where
$
    \boldsymbol{P}_\Psi = \boldsymbol{P}_d + \boldsymbol{B}_d \boldsymbol{\Psi}_d + \boldsymbol{A}_d \boldsymbol{\Pi}_c - \boldsymbol{\Pi}_c\boldsymbol{S}_d
$.
Since the continuous phase dynamics are diagonal, we have $\tilde{\boldsymbol{x}}(t_{k+1}^-) = \exp{(\boldsymbol{A}_cT)} \tilde{\boldsymbol{x}}(t_{k}^+)$ and $\boldsymbol{w}(t_{k+1}^-) = \exp{(\boldsymbol{S}_c T)} \boldsymbol{w}(t_{k}^+)$
which may be combined with the discrete phase of \eqref{eq: diagonal continuous phase} to obtain the step-to-step dynamics
\footnotesize
\begin{equation*}
    \begin{bmatrix}
    \tilde{\boldsymbol{x}}(t_{k+1}^+) \\ \boldsymbol{w}(t_{k+1}^+)
    \end{bmatrix}
    =
    \begin{bmatrix}
    \boldsymbol{A}_c\exp{(\boldsymbol{A}_cT)} & \boldsymbol{P}_\Psi \exp{(\boldsymbol{S}_c T)}
    \\
    \boldsymbol{0} & \boldsymbol{S}_d \exp{(\boldsymbol{S}_c T)} 
    \end{bmatrix}
    \begin{bmatrix}
    \tilde{\boldsymbol{x}}(t_{k}^+) \\ \boldsymbol{w}(t_{k}^+)
    \end{bmatrix}
\end{equation*}
\normalsize

This equation may be diagonalized to obtain $\boldsymbol{x}(t)$ as a function of $\boldsymbol{w}(t)$.
Consider the change of variables
$
    \Delta \tilde{\boldsymbol{x}}
    =
    \tilde{\boldsymbol{x}}
    -
    \boldsymbol{\Pi}_d \boldsymbol{w}
$
where $\boldsymbol{\Pi}_d$ is the solution to the Sylvester equation~\eqref{eq: Sylvester Equations for ss}.
Expressing the step-to-step dynamics using the new variables yields
\footnotesize
\begin{align*}
    \begin{bmatrix}
        \Delta\tilde{\boldsymbol{x}}(t_{k+1}^+) \\
        \boldsymbol{w}(t_{k+1}^+)
    \end{bmatrix}
    =
    \begin{bmatrix}
        \boldsymbol{A}_d \exp{(\boldsymbol{A}_cT)} & \boldsymbol{0} \\
        \boldsymbol{0} & \boldsymbol{S}_d \exp{(\boldsymbol{S}_c T)} 
    \end{bmatrix}
    \begin{bmatrix}
        \Delta\tilde{\boldsymbol{x}}(t_{k}^+) \\
        \boldsymbol{w}(t_{k}^+)
    \end{bmatrix}
\end{align*}
\normalsize
From the fact that $\boldsymbol{A}_d \exp{(\boldsymbol{A}_cT)}$ is assumed to be stable and $\boldsymbol{S}_d \exp{(\boldsymbol{S}_c T)}$ contains 2D rotation matrices along the diagonal, we conclude that:
$\Delta\tilde{\boldsymbol{x}}(t_{k}^+) \to \boldsymbol{0}$ as $k\to\infty$,
which is equivalent to
$\tilde{\boldsymbol{x}}(t_k^+) \to \boldsymbol{\Pi}_d \boldsymbol{w}(t_k^+)$ as $t\to\infty$.

Observing that for all $t$ in the interval $[t_k^+, t_{k+1}^-]$,
$
    \boldsymbol{x}(t)
    =
    \exp{\left(\boldsymbol{A}_c (t-t_k) \right)} \tilde{\boldsymbol{x}}(t_k^+)
    +
    \boldsymbol{\Pi}_c \boldsymbol{w}(t)
$, we conclude that, for all $t$ in the interval $[t_k^+, t_{k+1}^-]$, the steady-state trajectory of the system is~\eqref{eq: ss of hybrid sys signals}.


\subsection{Additional Lemma for Theorem~\ref{The: ctrl design}}
\begin{lemma}[Solvability of constrained Sylvester equation] \label{lem: comp design}
Let the pair $(\boldsymbol{F},\boldsymbol{H})$ be in observable canonical form with the characteristic polynomial $(s^m + f_1 s^{m-1} + f_2 s^{m-2} + \cdots + f_{m-1} s + f_m)^n$, let $\boldsymbol{S}\in\mathbb{R}^{p \times p}$, and let $\boldsymbol{\Psi}$ and $\boldsymbol{Q}$ be known matrices with conformable dimensions.
There exists at least one solution pair 
$(\boldsymbol{\Sigma},\boldsymbol{G})$
to the constrained Sylvester equation
\footnotesize
\begin{align}
\begin{split}
    \boldsymbol{\Sigma} \boldsymbol{S} =& \boldsymbol{F}\boldsymbol{\Sigma} + \boldsymbol{G}\boldsymbol{Q}
    \\
    \boldsymbol{H}\boldsymbol{\Sigma} =& \boldsymbol{\Psi}
\end{split} \label{eq: comp design lemma equations}
\end{align}
\normalsize
if $m \ge p$ and the pair $(\boldsymbol{S},\boldsymbol{Q})$ is observable.
\end{lemma}

\textit{Proof:}
To begin, let the partitions of $\boldsymbol{\Sigma}$ and $\boldsymbol{G}$ be $\boldsymbol{\Sigma}
    =
    \begin{bsmallmatrix}
        \boldsymbol{\sigma}_1 \\
        \boldsymbol{\sigma}_2 \\
        \vdots \\
        \boldsymbol{\sigma}_m
    \end{bsmallmatrix}$ and $\boldsymbol{G}
    =
    \begin{bsmallmatrix}
        \boldsymbol{g}_1 \\
        \boldsymbol{g}_2 \\
        \vdots \\
        \boldsymbol{g}_m
    \end{bsmallmatrix}$.
Then combining equations \eqref{eq: F and H canonical form} and \eqref{eq: comp design lemma equations} with rearrangement yields the system of equations
\footnotesize
\begin{align}
\begin{split}
    \boldsymbol{\sigma}_1 =& \boldsymbol{\Psi}
    \\
    \boldsymbol{\sigma}_2 =& \boldsymbol{\sigma}_1\boldsymbol{S} + f_1\boldsymbol{\sigma}_1 - \boldsymbol{g}_1\boldsymbol{Q}
    \\
    \vdots
    \\
    \boldsymbol{\sigma}_m =& \boldsymbol{\sigma}_{m-1}\boldsymbol{S} + f_{m-1}\boldsymbol{\sigma}_1 - \boldsymbol{g}_{m-1}\boldsymbol{Q}
    \\
    \boldsymbol{0} =& \boldsymbol{\sigma}_m\boldsymbol{S} + f_m\boldsymbol{\sigma}_1 - \boldsymbol{g}_{m}\boldsymbol{Q}
\end{split} \label{eq: how to solve for sigma}
\end{align}
\normalsize
Observe that substituting the first row into the second row yields
$
\boldsymbol{\sigma}_2 = \boldsymbol{\Psi}\boldsymbol{S} + f_1\boldsymbol{\Psi} - \boldsymbol{g}_1\boldsymbol{Q}
$.
Substituting this result into the third row yields
$
    \boldsymbol{\sigma}_3 = \boldsymbol{\Psi}\boldsymbol{S}^2 + f_1\boldsymbol{\Psi}\boldsymbol{S} + f_2\boldsymbol{\Psi} - \boldsymbol{g}_1\boldsymbol{Q}\boldsymbol{S} - \boldsymbol{g}_2\boldsymbol{Q}
$.
Proof by induction may be used to show that repeating the operation until the last equation is reached yields
$
    \boldsymbol{0} = \boldsymbol{\Psi}\boldsymbol{S}^m + \sum_{k=1}^m f_k\boldsymbol{\Psi}\boldsymbol{S}^{m-k} - \sum_{k=1}^m \boldsymbol{g}_k\boldsymbol{Q}\boldsymbol{S}^{m-k}
$
which can be rearranged and written in matrix form as
\footnotesize
\begin{align}
    \begin{bmatrix}
        \boldsymbol{g}_{1}^T \\ \boldsymbol{g}_{2}^T \\ \cdots \\ \boldsymbol{g}_{m-1}^T \\ \boldsymbol{g}_{m}^T
    \end{bmatrix}^T
    \begin{bmatrix}
        \boldsymbol{Q}\boldsymbol{S}^{m-1}
        \\
        \boldsymbol{Q}\boldsymbol{S}^{m-2}
        \\
        \vdots
        \\
        \boldsymbol{Q}\boldsymbol{S}
        \\
        \boldsymbol{Q}
    \end{bmatrix}
    =
    \boldsymbol{\Psi}\boldsymbol{S}^m + \sum_{k=1}^m f_k\boldsymbol{\Psi}\boldsymbol{S}^{m-k}
\end{align}
\normalsize
When $m \ge p$, the lower part of the matrix is the observability matrix corresponding to the pair $(\boldsymbol{S},\boldsymbol{Q})$.
Therefore, at least one $\boldsymbol{G}$ will exist when $m \ge p$ and the pair $(\boldsymbol{S},\boldsymbol{Q})$ is observable \cite{hespanha2018linear}.  The solution $\boldsymbol{G}$ may be used to obtain $\boldsymbol{\Sigma}$ using \eqref{eq: how to solve for sigma}.

\subsection{Proof of Theorem~\ref{The: ctrl design}}

To begin, the existence of the parameters of the discrete phase compensator that reject the disturbance is analyzed as follows.
To analyze the steady-state behavior of the discrete phase compensator, we construct a subsystem composed of the discrete phase compensator~\eqref{eq: disc comp} and the dynamics of $\boldsymbol{d}(t)$ in~\eqref{eq: d generator}. The subsystem's state is defined as $[\boldsymbol{\eta}_d^T, \tilde{\boldsymbol{e}}^T, \boldsymbol{w}^T]^T$. 
The expression is omitted here.
By applying Lemma~\ref{lem: ss of hybrid sys.} to this subsystem, we have, for $t \in [t_k^+,t_{k+1}^-]$, the steady-state trajectories
\footnotesize
\begin{align}
\begin{split}
    \boldsymbol{\eta}_d(t) =& \boldsymbol{\Sigma}_d \boldsymbol{w}(t_k^+) \\
    \tilde{\boldsymbol{e}}(t) =& \exp{(\boldsymbol{A}_c(t-t_k)})\boldsymbol{\Pi}_e \boldsymbol{w}(t_k^+) \\
    \boldsymbol{u}_d(t) =& \boldsymbol{H}_d \boldsymbol{\Sigma}_d\boldsymbol{w}(t_k^+)
\end{split}
\end{align}
\normalsize
where $\boldsymbol{\Sigma}_d$ and $\boldsymbol{\Pi}_e$ are part of the solution to the Sylvester equation
\footnotesize
\begin{align}
\begin{split}
    \begin{bmatrix}
        \boldsymbol{\Sigma}_d \\ \boldsymbol{\Pi}_e
    \end{bmatrix}
    \boldsymbol{S}_d \exp{(\boldsymbol{S}_cT)}
    =&
    \begin{bmatrix}
        \boldsymbol{F}_d & \boldsymbol{G}_d\boldsymbol{C}\exp(\boldsymbol{A}_cT) \\
        \boldsymbol{0} & \boldsymbol{A}_d\exp(\boldsymbol{A}_cT)
    \end{bmatrix}
    \begin{bmatrix}
        \boldsymbol{\Sigma}_d \\ \boldsymbol{\Pi}_e
    \end{bmatrix}
    \\ +&
    \begin{bmatrix}
        \boldsymbol{G}_d \boldsymbol{Q}_0 \\ \boldsymbol{P}_0
    \end{bmatrix}
    \exp{(\boldsymbol{S}_cT)}
\end{split}
\end{align}
\normalsize
The upper diagonal structure of the Sylvester equation allows it to be separated into \eqref{eq: Pie} and the Sylvester equation
\footnotesize
\begin{align}
    \boldsymbol{\Sigma}_d \boldsymbol{S}_d \exp{(\boldsymbol{S}_cT)}
    =
    \boldsymbol{F}_d \boldsymbol{\Sigma}_d
    +
    \boldsymbol{G}_d
    \boldsymbol{Q}_e \label{eq: jump Sylvester equation}
\end{align}
\normalsize
The discrete phase compensator produces a control input which is constant across the continuous phase
    $\boldsymbol{u}_d(t_{k+1}^-)
    =
    \boldsymbol{H}_d\boldsymbol{\eta}(t_{k+1}^-)
    =
    \boldsymbol{H}_d\boldsymbol{\eta}(t_k^+)
    =
    \boldsymbol{H}_d\boldsymbol{\Sigma}_d\boldsymbol{w}(t_k^+)$
and the input will reject the disturbance if
\footnotesize
\begin{align}
    \boldsymbol{H}_d\boldsymbol{\Sigma}_d\boldsymbol{w}(t_k^+)
    =
    \boldsymbol{\Psi}_d \boldsymbol{w}(t_{k+1}^-)
    =
    \boldsymbol{\Psi}_d \exp(\boldsymbol{S}_cT) \boldsymbol{w}(t_k^+) \label{eq: jump constraint}
\end{align}
\normalsize
Combining \eqref{eq: jump Sylvester equation} and \eqref{eq: jump constraint} yields
\footnotesize
\begin{align}
\begin{split} \label{eq: disc comp constrained Sylvester equation}
    \boldsymbol{\Sigma}_d \boldsymbol{S}_d \exp{(\boldsymbol{S}_cT)}
    =&
    \boldsymbol{F}_d \boldsymbol{\Sigma}_d
    +
    \boldsymbol{G}_d
    \boldsymbol{Q}_e
    \\
    \boldsymbol{H}_d\boldsymbol{\Sigma}_d =& \boldsymbol{\Psi}_d\exp(\boldsymbol{S}_cT)
\end{split}
\end{align}
\normalsize
Lemma \ref{lem: comp design} can then be applied to show that there exists a solution $(\boldsymbol{\Sigma}_d,\boldsymbol{G}_d)$
if
$n_d \ge \text{dim}(\boldsymbol{w})$
and the pair 
$(\boldsymbol{S}_d\exp(\boldsymbol{S}_cT),\boldsymbol{Q}_e)$
is observable.

Similarly, to analyze the steady-state behavior of the continuous and jump compensators, we construct a subsystem composed of~\eqref{eq: cont comp} and \eqref{eq: d generator} and whose state is defined as $[\boldsymbol{\eta}_c^T, \boldsymbol{\eta}_j^T, \tilde{\boldsymbol{e}}^T,\boldsymbol{w}^T]^T$.
The expression is omitted here for simplicity. 
Since the control system designer has chosen $\boldsymbol{F}_c$ such that $\sigma(\boldsymbol{A}_c) \cap \sigma(\boldsymbol{F}_c) = \emptyset$, it is possible to diagonalize the continuous phase of the sub-system \cite{isidori2017lectures} using the change of variables
    $\tilde{\boldsymbol{\eta}}_c
    =
    \boldsymbol{\eta}_c
    -
    \begin{bmatrix}
        \boldsymbol{\Sigma}_e & \boldsymbol{\Sigma}_w
    \end{bmatrix}
    \begin{bmatrix}
        \tilde{\boldsymbol{e}} \\ \boldsymbol{w}
    \end{bmatrix}$
where $\begin{bmatrix} \boldsymbol{\Sigma}_e & \boldsymbol{\Sigma}_w \end{bmatrix}$ is the solution to the Sylvester equation
\footnotesize
\begin{align} \label{eq: cont comp Sylvester}
    \begin{bmatrix}
        \boldsymbol{\Sigma}_e & \boldsymbol{\Sigma}_w
    \end{bmatrix}
    \begin{bmatrix}
        \boldsymbol{A}_c & \boldsymbol{0} \\
        \boldsymbol{0} & \boldsymbol{S}_c
    \end{bmatrix}
    =
    \boldsymbol{F}_c
    \begin{bmatrix}
        \boldsymbol{\Sigma}_e & \boldsymbol{\Sigma}_w
    \end{bmatrix}
    +
    \boldsymbol{G}_c
    \begin{bmatrix}
        \boldsymbol{C} & \boldsymbol{Q}_0
    \end{bmatrix}
\end{align}
\normalsize
In the new coordinates, the sub-system is described by
\footnotesize
\begin{align}
\begin{split}
    \begin{bmatrix}
        \dot{\tilde{\boldsymbol{\eta}}}_c \\ \dot{\boldsymbol{\eta}}_j \\ \dot{\tilde{\boldsymbol{e}}} \\ \dot{\boldsymbol{w}}
    \end{bmatrix}
    &=
    \begin{bmatrix}
        \boldsymbol{F}_c & \boldsymbol{0} & \boldsymbol{0} & \boldsymbol{0} \\
        \boldsymbol{0} & \boldsymbol{0} & \boldsymbol{0} & \boldsymbol{0} \\
        \boldsymbol{0} & \boldsymbol{0} & \boldsymbol{A}_c & \boldsymbol{0} \\
        \boldsymbol{0} & \boldsymbol{0} & \boldsymbol{0} & \boldsymbol{S}_c
    \end{bmatrix}
    \begin{bmatrix}
        \tilde{\boldsymbol{\eta}}_c \\\boldsymbol{\eta}_j \\ \tilde{\boldsymbol{e}} \\ \boldsymbol{w}
    \end{bmatrix}
    \\
    \begin{bmatrix}
        \tilde{\boldsymbol{\eta}}_c(t_k^+) \\\boldsymbol{\eta}_j(t_k^+) \\ \tilde{\boldsymbol{e}}(t_k^+) \\ \boldsymbol{w}(t_k^+)
    \end{bmatrix}
    &=
    \begin{bmatrix}
        \boldsymbol{I} & \boldsymbol{H}_j & \boldsymbol{P}_{he} & \boldsymbol{P}_{hw} \\
        \boldsymbol{0} & \boldsymbol{F}_j & \boldsymbol{G}_j\boldsymbol{C} & \boldsymbol{G}_j\boldsymbol{Q}_0 \\
        \boldsymbol{0} & \boldsymbol{0} & \boldsymbol{A}_d & \boldsymbol{P}_0 \\
        \boldsymbol{0} & \boldsymbol{0} & \boldsymbol{0} & \boldsymbol{S}_d
    \end{bmatrix}
    \begin{bmatrix}
        \tilde{\boldsymbol{\eta}}_c(t_k^-) \\ \boldsymbol{\eta}_j(t_k^-) \\ \tilde{\boldsymbol{e}}(t_k^-) \\ \boldsymbol{w}(t_k^-)
    \end{bmatrix}
    \\
    \boldsymbol{u}_c
    &=
    \begin{bmatrix}
        \boldsymbol{H}_c & \boldsymbol{0} & \boldsymbol{H}_c \boldsymbol{\Sigma}_e & \boldsymbol{H}_c \boldsymbol{\Sigma}_w
    \end{bmatrix}
    \begin{bmatrix}
        \tilde{\boldsymbol{\eta}}_c \\ \boldsymbol{\eta}_j \\ \tilde{\boldsymbol{e}} \\ \boldsymbol{w}
    \end{bmatrix}
\end{split}
\end{align}
\normalsize
where $\boldsymbol{P}_{he} = \boldsymbol{\Sigma}_e - \boldsymbol{\Sigma}_e\boldsymbol{A}_d$ and $\boldsymbol{P}_{hw} = \boldsymbol{\Sigma}_w - \boldsymbol{\Sigma}_w\boldsymbol{S}_d - \boldsymbol{\Sigma}_e\boldsymbol{P}_0$.
Then Lemma \ref{lem: ss of hybrid sys.} may be used to show that, for $t \in [t_k^+,t_{k+1}^-]$, we have the steady-state trajectories
\footnotesize
\begin{align}
\begin{split}
    \tilde{\boldsymbol{\eta}}_c(t) =& \exp{(\boldsymbol{F}_c(t-t_k))} \boldsymbol{\Sigma}_c \boldsymbol{w}(t_k^+) \\
    \boldsymbol{\eta}_j(t) =& \boldsymbol{\Sigma}_j \boldsymbol{w}(t_k^+) \\
    \boldsymbol{\tilde{e}}(t) =& \exp{(\boldsymbol{A}_c(t-t_k))} \boldsymbol{\Pi}_e \boldsymbol{w}(t_k^+) \\
    \boldsymbol{u}_c(t) =& \boldsymbol{H}_c \exp{(\boldsymbol{F}_c(t-t_k))} \boldsymbol{\Sigma}_c \boldsymbol{w}(t_k^+)
    \\ +&
    \boldsymbol{H}_c \boldsymbol{\Sigma}_e \exp{(\boldsymbol{A}_c(t-t_k))} \boldsymbol{\Pi}_e \boldsymbol{w}(t_k^+)
    \\ +&
    \boldsymbol{H}_c \boldsymbol{\Sigma}_w \boldsymbol{w}(t)
\end{split} \label{eq: cont comp signal}
\end{align}
\normalsize
where $\boldsymbol{\Sigma}_c$, $\boldsymbol{\Sigma}_j$, and $\boldsymbol{\Pi}_e$ are part of the solution to the Sylvester equation
\footnotesize
$\begin{bmatrix}
        \boldsymbol{\Sigma}_c \\ \boldsymbol{\Sigma}_j \\ \boldsymbol{\Pi}_e
    \end{bmatrix}
    \boldsymbol{S}_d \exp{(\boldsymbol{S}_cT)}
    =
    \begin{bmatrix}
        \exp{(\boldsymbol{F}_cT)} & \boldsymbol{H}_j & \boldsymbol{P}_{he} \exp{(\boldsymbol{A}_cT)} \\
        \boldsymbol{0} & \boldsymbol{F}_j & \boldsymbol{G}_j\boldsymbol{C} \exp{(\boldsymbol{A}_cT)} \\
        \boldsymbol{0} & \boldsymbol{0} & \boldsymbol{A}_d \exp(\boldsymbol{A}_cT)
    \end{bmatrix}
    \begin{bmatrix}
        \boldsymbol{\Sigma}_c \\ \boldsymbol{\Sigma}_j \\ \boldsymbol{\Pi}_e
    \end{bmatrix}
     +
    \begin{bmatrix}
        \boldsymbol{P}_{hw} \\ \boldsymbol{G}_j\boldsymbol{Q}_0 \\ \boldsymbol{P}_0
    \end{bmatrix}
    \exp{(\boldsymbol{S}_cT)}$
\normalsize
which can be separated into (\ref{eq: Pie}) and
\footnotesize
\begin{align} 
    \label{eq: trans comp Sylvester equation}
    \boldsymbol{\Sigma}_j \boldsymbol{S}_d \exp{(\boldsymbol{S}_cT)}
    =&
    \boldsymbol{F}_j \boldsymbol{\Sigma}_j
    +
    \boldsymbol{G}_j\boldsymbol{Q}_e
    \\
    \boldsymbol{\Sigma}_c \boldsymbol{S}_d \exp{(\boldsymbol{S}_cT)}
    =&
    \exp{(\boldsymbol{F}_cT)} \boldsymbol{\Sigma}_c
    +
    \boldsymbol{H}_j\boldsymbol{\Sigma}_j - \boldsymbol{\Psi}_j
\end{align}
\normalsize
where
$
    \boldsymbol{\Psi}_j
    =
    -\left( \boldsymbol{P}_{he} \exp{(\boldsymbol{A}_cT)} \boldsymbol{\Pi}_e
    +
    \boldsymbol{P}_{hw} \exp{(\boldsymbol{S}_cT)} \right)
$.

The steady-state trajectory of $\boldsymbol{u}_c$ is $\boldsymbol{\Psi}_c\boldsymbol{w}(t)$, the control input which rejects the disturbance (recall Theorem~\ref{The: hybird regulator equation}), if
\footnotesize
\begin{align}
\begin{split} \label{eq: cont comp constraints}
    \boldsymbol{H}_c \begin{bmatrix}
        \boldsymbol{\Sigma}_e & \boldsymbol{\Sigma}_w
    \end{bmatrix}
    =&
    \begin{bmatrix}
        \boldsymbol{0} & \boldsymbol{\Psi}_c
    \end{bmatrix}
    \\
    \boldsymbol{H}_j\boldsymbol{\Sigma}_j = \boldsymbol{\Psi}_j
\end{split}
\end{align}
\normalsize

Collecting \eqref{eq: cont comp Sylvester}, \eqref{eq: trans comp Sylvester equation}, and \eqref{eq: cont comp constraints} yields
\footnotesize
\begin{align}
\begin{split} \label{eq: cont comp constrained Sylvester equation}
    \begin{bmatrix}
        \boldsymbol{\Sigma}_e & \boldsymbol{\Sigma}_w
    \end{bmatrix}
    \begin{bmatrix}
        \boldsymbol{A}_c & \boldsymbol{0} \\
        \boldsymbol{0} & \boldsymbol{S}_c
    \end{bmatrix}
    =&
    \boldsymbol{F}_c
    \begin{bmatrix}
        \boldsymbol{\Sigma}_e & \boldsymbol{\Sigma}_w
    \end{bmatrix}
    +
    \boldsymbol{G}_c
    \begin{bmatrix}
        \boldsymbol{C} & \boldsymbol{Q}_0
    \end{bmatrix}
    \\
    \boldsymbol{H}_c \begin{bmatrix}
        \boldsymbol{\Sigma}_e & \boldsymbol{\Sigma}_w
    \end{bmatrix}
    =&
    \begin{bmatrix}
        \boldsymbol{0} & \boldsymbol{\Psi}_c
    \end{bmatrix}
\end{split}
\end{align}
\normalsize
and
\footnotesize
\begin{align}
\begin{split} \label{eq: trans comp constrained Sylvester equation}
    \boldsymbol{\Sigma}_j \boldsymbol{S}_d \exp{(\boldsymbol{S}_cT)}
    =&
    \boldsymbol{F}_j \boldsymbol{\Sigma}_j
    +
    \boldsymbol{G}_j\boldsymbol{Q}_e
    \\
    \boldsymbol{H}_j\boldsymbol{\Sigma}_j =& \boldsymbol{\Psi}_j
\end{split}
\end{align}
\normalsize
Lemma \ref{lem: comp design} can be used to show that there exists a solution pair 
$(\begin{bmatrix} \boldsymbol{\Sigma}_e & \boldsymbol{\Sigma}_w \end{bmatrix} , \boldsymbol{G}_c)$
if 
$m_c \ge \text{dim}(\boldsymbol{x}) + \text{dim}(\boldsymbol{w})$
and the pair 
$( \text{diag}( \boldsymbol{A}_c,\boldsymbol{S}_c) , \begin{bmatrix} \boldsymbol{C} & \boldsymbol{Q}_0 \end{bmatrix} )$
is observable.
Similarly, there exists a solution pair 
$(\boldsymbol{\Sigma}_j , \boldsymbol{G}_j)$
if 
$m_j \ge \text{dim}(\boldsymbol{w})$
and the pair 
$(\boldsymbol{S}_d\exp(\boldsymbol{S}_cT),\boldsymbol{Q}_e)$
is observable.

\subsection{Proof of Lemma \ref{lem: hybrid swapping}}
To simplify the expression of~\eqref{eq: controller subsystem2}, we define
\footnotesize
$$
\boldsymbol{\Xi}_D(t_k^-) = \begin{bmatrix}
    \boldsymbol{\Xi}_j(t_k^-) & \boldsymbol{0} \\
    \boldsymbol{0} & \boldsymbol{\Xi}_d(t_k^-)
\end{bmatrix}, \quad
\hat{\boldsymbol{\theta}}_D(t_k^-)=\begin{bmatrix}
    \hat{\boldsymbol{\theta}}_j(t_k^-) \\
    \hat{\boldsymbol{\theta}}_d(t_k^-)
\end{bmatrix}
$$
\normalsize
To provide a more general lemma, $\bar{\boldsymbol{A}}_c, \bar{\boldsymbol{A}}_d, \bar{\boldsymbol{B}}_c, \bar{\boldsymbol{B}}_d, \bar{\boldsymbol{C}}$ are assumed to be time-varying and continuous in this proof so the lemma can be generalized to time-varying hybrid linear systems.

Observe that the solution to \eqref{eq: controller subsystem2} over the interval $[t_k^+,t_{k+1}^-]$ is
$
\hat{\boldsymbol{y}}(t) = \bar{\boldsymbol{C}}(t)\boldsymbol{\Phi}(t,t_k^+)\boldsymbol{z}(t_k^+) + \bar{\boldsymbol{C}}(t)\mathcal{H}_0[ \bar{\boldsymbol{B}}_c \boldsymbol{\Xi}_c^T\hat{\boldsymbol{\theta}}_c ](t,t_k^+)
$
where $\boldsymbol{\Phi}(\cdot,\cdot)$ is the state-transition matrix corresponding to $\bar{\boldsymbol{A}}_c(t)$~\cite{hespanha2018linear} and 
$\mathcal{H}_0[ \bar{\boldsymbol{B}}_c \boldsymbol{\Xi}_c^T\hat{\boldsymbol{\theta}}_c ](t,t_k^+) = \int\nolimits_{t_k^+}^{t}\boldsymbol{\Phi}(t,\tau)\bar{\boldsymbol{B}}_c(\tau)\boldsymbol{\Xi}_c^T(\tau)\hat{\boldsymbol{\theta}}_c(\tau) \partial \tau$.
Next, observe that the discrete phase of \eqref{eq: controller subsystem2} may be used to rewrite $\hat{\boldsymbol{y}}(t)$ as
$
\hat{\boldsymbol{y}}(t)
=
\bar{\boldsymbol{C}}(t)\boldsymbol{\Phi}(t,t_k^+)
\bar{\boldsymbol{A}}_d(t_k)\boldsymbol{z}(t_k^-)
 +
\bar{\boldsymbol{C}}(t)\boldsymbol{\Phi}(t,t_k^+)
\bar{\boldsymbol{B}}_d(t_k)\boldsymbol{\Xi}_D^T(t_k^-)
\hat{\boldsymbol{\theta}}_D(t_k^-)
 +
\bar{\boldsymbol{C}}(t)\mathcal{H}_0[ \bar{\boldsymbol{B}}_c \boldsymbol{\Xi}_c^T\hat{\boldsymbol{\theta}}_c ](t,t_k^+)
$.

Proof by induction can be used to show that repeating the above procedures allows $\hat{\boldsymbol{y}}(t)$ to be expressed as a function of the inputs and the initial state $\boldsymbol{z}(t_0)$, assuming $t_0=t_0^+$. That is,
$
    \hat{\boldsymbol{y}}(t)
    =
    \boldsymbol{\varepsilon}_0(t)
    +
    \mathcal{G}_c[ \mathcal{H}_0 [\bar{\boldsymbol{B}}_c \boldsymbol{\Xi}_c^T\hat{\boldsymbol{\theta}}_c ] ](t)
    +
    \mathcal{G}_d[ \bar{\boldsymbol{B}}_d \boldsymbol{\Xi}_D^T\hat{\boldsymbol{\theta}}_D ](t)
$
where $\boldsymbol{\varepsilon}_0(t) = \bar{\boldsymbol{C}}(t) \boldsymbol{\Phi}(t,t_k^+) \boldsymbol{\psi}(t_k^+,t_0) \boldsymbol{z}(t_0)$ is the effect of the initial condition, $\boldsymbol{\psi}(t_k^+,t_i^+) = \Pi_{j=i+1}^k \bar{\boldsymbol{A}}_d(t_j)\boldsymbol{\Phi}(t_j^-,t_{j-1}^+)$ is the state transition matrix from $t_i^+$ to $t_k^+$, 
    $\mathcal{G}_c [ 
    \mathcal{H}_0 [  
    \bar{\boldsymbol{B}}_c \boldsymbol{\Xi}_c^T\hat{\boldsymbol{\theta}}_c 
    ] 
    ](t)
    = 
    \bar{\boldsymbol{C}}(t)
    \mathcal{H}_0
    [ 
    \bar{\boldsymbol{B}}_c \boldsymbol{\Xi}_c^T\hat{\boldsymbol{\theta}}_c 
    ](t,t_k^+) + \mathcal{G}_d[ \bar{\boldsymbol{A}}_d\mathcal{H}_0[ \bar{\boldsymbol{B}}_c \boldsymbol{\Xi}_c^T\hat{\boldsymbol{\theta}}_c] ](t)$ captures the effect of the continuous control inputs from $t_0$ to $t$, 
    $\mathcal{G}_d[ \bar{\boldsymbol{B}}_d \boldsymbol{\Xi}_D^T\hat{\boldsymbol{\theta}}_D ](t)
    = \bar{\boldsymbol{C}}(t)\boldsymbol{\Phi}(t,t_k^+)\bar{\boldsymbol{B}}_d(t_k)\boldsymbol{\Xi}_D^T(t_k^-)\hat{\boldsymbol{\theta}}_D(t_k^-) + \bar{\boldsymbol{C}}(t)\boldsymbol{\Phi}(t,t_k^+)\mathcal{F}_0[ \bar{\boldsymbol{B}}_d \boldsymbol{\Xi}_D^T\hat{\boldsymbol{\theta}}_D ](t_k^+)$ describes the effect of the discrete control inputs from $t_1^-$ to $t_k^-$, and
    $\mathcal{F}_0[ \bar{\boldsymbol{B}}_d \boldsymbol{\Xi}_D^T\hat{\boldsymbol{\theta}}_D ](t_k^+)
    = \sum\nolimits_{i=1}^{k-1}\boldsymbol{\psi}(t_k^+,t_i^+)\bar{\boldsymbol{B}}_d(t_i)\boldsymbol{\Xi}_D^T(t_i^-)\hat{\boldsymbol{\theta}}_D^T(t_i^-)$ simplifies the expression and captures the effect of control inputs before $t_{k-1}^-$ on $\hat{\boldsymbol{y}}(t_k^+)$. 

Next, the discrete-time swapping lemma~\cite{Ioannou2006tutorial} may be applied to obtain
$\mathcal{G}_d[ \bar{\boldsymbol{B}}_d \boldsymbol{\Xi}_D^T\hat{\boldsymbol{\theta}}_D ](t)
=
\mathcal{G}_d[ \bar{\boldsymbol{B}}_d \boldsymbol{\Xi}_D^T ](t) \hat{\boldsymbol{\theta}}_D(t)
 -
\boldsymbol{\varepsilon}_1[\Delta\hat{\boldsymbol{\theta}}_D](t)$
where 
$\boldsymbol{\varepsilon}_1 [\Delta \hat{\boldsymbol{\theta}}_D ] (t) = \mathcal{G}_d [\bar{\boldsymbol{B}}_d \boldsymbol{\Xi}_D^T ] (t)(\hat{\boldsymbol{\theta}}_D(t) - \hat{\boldsymbol{\theta}}_D(t_k^-)) 
\quad + \bar{\boldsymbol{C}}(t) \boldsymbol{\Phi}(t,t_k^+) \mathcal{F}_0 [\mathcal{F}_0 [\bar{\boldsymbol{B}}_d \boldsymbol{\Xi}_D^T ] \Delta \hat{\boldsymbol{\theta}}_D ] (t_k^+)$ is the swapping term and $\Delta\hat{\boldsymbol{\theta}}_D(t_j^-) = \hat{\boldsymbol{\theta}}_D(t_{j+1}^-)-\hat{\boldsymbol{\theta}}_D(t_j^-)$ is the change of the controller parameters.

Similarly, the continuous-time swapping lemma~\cite{ioannou1993ltvSystems}, the linearity and definition of $\mathcal{G}_c$, and the discrete-time swapping lemma~\cite{Ioannou2006tutorial} may be applied to obtain
$
    \mathcal{G}_c [ 
    \mathcal{H}_0 [  
    \bar{\boldsymbol{B}}_c \boldsymbol{\Xi}_c^T\hat{\boldsymbol{\theta}}_c 
    ] 
    ](t)
    =
    \mathcal{G}_c [ 
    \mathcal{H}_0 [  
    \bar{\boldsymbol{B}}_c \boldsymbol{\Xi}_c^T 
    ] 
    ](t)
    \hat{\boldsymbol{\theta}}_c(t)
     -
    \boldsymbol{\varepsilon}_2[\Delta\hat{\boldsymbol{\theta}}_c](t)
    -
    \boldsymbol{\varepsilon}_3[\dot{\hat{\boldsymbol{\theta}}}_c](t)
$
where 
$\boldsymbol{\varepsilon}_2 [\Delta \hat{\boldsymbol{\theta}}_c ] (t) = \mathcal{G}_d [\bar{\boldsymbol{A}}_d \mathcal{H}_0 [\bar{\boldsymbol{B}}_c \boldsymbol{\Xi}_c^T ] ] (t)(\hat{\boldsymbol{\theta}}_c(t) - \hat{\boldsymbol{\theta}}_c(t_k^-)) 
\quad + \bar{\boldsymbol{C}}(t) \boldsymbol{\Phi}(t,t_k^+) \mathcal{F}_0 [\mathcal{F}_0 [\bar{\boldsymbol{A}}_d \mathcal{H}_0 [\bar{\boldsymbol{B}}_c \boldsymbol{\Xi}_c^T ] ] \Delta \hat{\boldsymbol{\theta}}_c ] (t)$ and
$\boldsymbol{\varepsilon}_3 [\dot{\hat{\boldsymbol{\theta}}}_c ] (t) = \mathcal{G}_c [\mathcal{H}_0 [\mathcal{H}_0 [\bar{\boldsymbol{B}}_c \boldsymbol{\Xi}_c^T ] \dot{\hat{\boldsymbol{\theta}}}_c ] ] (t)$ are the swapping terms and
$\Delta\hat{\boldsymbol{\theta}}_c(t_j^-) = \hat{\boldsymbol{\theta}}_c(t_{j+1}^-)-\hat{\boldsymbol{\theta}}_c(t_j^-)$ is the change of the controller parameters.
Combining these results yields
    $\hat{\boldsymbol{y}}(t)
    =
    \mathcal{G}_c[ \mathcal{H}_0[ \bar{\boldsymbol{B}}_c \boldsymbol{\Xi}_c^T ] ](t)\hat{\boldsymbol{\theta}}_c(t)
    +
    \mathcal{G}_d[ \bar{\boldsymbol{B}}_d \boldsymbol{\Xi}_D^T ](t)\hat{\boldsymbol{\theta}}_D(t)
    +
    \boldsymbol{\varepsilon}_0(t)
    -
    \boldsymbol{\varepsilon}_1[\Delta\hat{\boldsymbol{\theta}}_d](t)
    -
    \boldsymbol{\varepsilon}_2[\Delta\hat{\boldsymbol{\theta}}_c](t)
    -
    \boldsymbol{\varepsilon}_3[\dot{\hat{\boldsymbol{\theta}}}_c](t)$.
To complete the proof, let $\boldsymbol{\varepsilon}(t)=\boldsymbol{\varepsilon}_0(t) -\boldsymbol{\varepsilon}_1[\Delta\hat{\boldsymbol{\theta}}_d](t)-\boldsymbol{\varepsilon}_2[\Delta\hat{\boldsymbol{\theta}}_c](t)-\boldsymbol{\varepsilon}_3[\dot{\hat{\boldsymbol{\theta}}}_c](t)$, which goes to zero when the parameters stop changing and $t\rightarrow\infty$, and observe that $\mathcal{G}_c[ \mathcal{H}_0[ \bar{\boldsymbol{B}}_c \boldsymbol{\Xi}_c^T ] ](t)$ and $\mathcal{G}_d[ \bar{\boldsymbol{B}}_d \boldsymbol{\Xi}_D^T ](t)$ are the outputs of the systems~\eqref{eq: continuous regressor matrix generator}~and~\eqref{eq: discrete regressor matrix generator}, respectively.
\normalsize

\subsection{Proof of Theorem \ref{The: adaptive ctrl performance}}
To begin, we show that $\boldsymbol{K}(t)$ is uniformly bounded from above and below.
Recall that symmetric and positive definite matrices $\boldsymbol{K}(t)$ permit an orthonormal eigenvalue decomposition \cite{horn1990matrix}, which can be used to show that
$
    0 \le \tfrac{\boldsymbol{K}(t)\boldsymbol{\phi}^T(t)\boldsymbol{\phi}(t)\boldsymbol{K}(t)}{1+\text{trace}[\boldsymbol{\phi}^T(t)\boldsymbol{K}(t)\boldsymbol{\phi}(t)]} \le \boldsymbol{K}(t)
$
and
$
    f_L(\boldsymbol{K}(t))
    \le
    \dot{\boldsymbol{K}}(t)
    \le
    f_U(\boldsymbol{K}(t))
$
where the lower and upper bounding quadratic polynomials are given by $f_L(\boldsymbol{K}(t)) = \beta \boldsymbol{K}^0(t) + (\gamma - \alpha)\boldsymbol{K}(t) - \delta \boldsymbol{K}^2(t)$ and $f_U(\boldsymbol{K}(t)) = \beta \boldsymbol{K}^0(t) + \gamma\boldsymbol{K}(t) - \delta \boldsymbol{K}^2(t)$.
Letting $\mu_i(t)$ and $\boldsymbol{v}_i(t)$ be the $i^{\text{th}}$ eigenvalue and eigenvector of $\boldsymbol{K}(t)$ respectively, we obtain $f_L(\boldsymbol{K}(t)) = \sum_{\mu_i(t)\in\sigma(\boldsymbol{K}(t))}f_L(\mu_i(t))\boldsymbol{v}_i(t)\boldsymbol{v}_i^T(t)$ and $f_U(\boldsymbol{K}(t)) = \sum_{\mu_i(t)\in\sigma(\boldsymbol{K}(t))}f_U(\mu_i(t))\boldsymbol{v}_i(t)\boldsymbol{v}_i^T(t)$
and
$
    f_L(\mu_i(t)) \le \dot{\mu}_i(t) \le f_U(\mu_i(t))
$.
Observe that $\dot{\mu}_i \le f_U(\mu_i)<0$ for $\mu_i >\overline{\sigma}$ implies that the $i^{\text{th}}$ eigenvalue of $\boldsymbol{K}(t)$ is guaranteed to stop growing as $\mu_i(t)$ approaches $\overline{\sigma}$.
Similarly, $\dot{\mu}_i > f_L(\mu_i) > 0$ for $\mu_i\in[0,\underline{\sigma})$ implies that the $i^{\text{th}}$ eigenvalue of $\boldsymbol{K}(t)$ is guaranteed to stop shrinking as $\mu_i(t)$ approaches $\underline{\sigma}$.
Thus we conclude that $\underline{\sigma}\boldsymbol{I} \le \boldsymbol{K}(t) \le \overline{\sigma}\boldsymbol{I}$ for all $t$ provided the initial condition is chosen such that $\underline{\sigma}\boldsymbol{I} \le \boldsymbol{K}(0) \le \overline{\sigma}\boldsymbol{I}$.

Next, let the normalization signal be
$
    n_s^2(t) = (1 + \text{trace}(\boldsymbol{\phi}^T(t)\boldsymbol{K}(t)\boldsymbol{\phi}(t)))
$
and observe that, using~\eqref{eq: def d(t)} and \eqref{eq: estimated output}, $\boldsymbol{y}(t)$ may be re-written as
    $\boldsymbol{y}(t)
    =
    n_s^2(t)\boldsymbol{\epsilon}(t)
    -
    \begin{bmatrix}
         \boldsymbol{\phi}_j^T(t) & \boldsymbol{\phi}_d^T(t)
    \end{bmatrix}
    (
    \begin{bmatrix}
         \hat{\boldsymbol{\theta}}_j(t) \\ \hat{\boldsymbol{\theta}}_d(t)
    \end{bmatrix}
    -
    \begin{bmatrix}
         \hat{\boldsymbol{\theta}}_j(t_k^+) \\ \hat{\boldsymbol{\theta}}_d(t_k^+)
    \end{bmatrix}) 
    +
    \boldsymbol{\varepsilon}(t)$
Therefore, the measured output will converge to zero if the adaptive law is designed to drive the estimation error $\boldsymbol{\epsilon}(t)$ to zero and stop adaptation.
Consider the Lyapunov candidate
$
    V = \tilde{\boldsymbol{\theta}}^T(t)\boldsymbol{K}^{-1}(t)\tilde{\boldsymbol{\theta}}(t)
$
where
$
    \tilde{\boldsymbol{\theta}}(t) = \hat{\boldsymbol{\theta}}(t) - \boldsymbol{\theta}^*
$
and where $\boldsymbol{\theta}^*$ is one of the constant solutions described by Theorem \ref{The: ctrl design} such that $|\boldsymbol{\theta}^*| \le \theta_m$.

Using the identity
$
    \frac{\partial}{\partial t}\left[ \boldsymbol{K}^{-1}(t) \right] = - \boldsymbol{K}^{-1}(t)\dot{\boldsymbol{K}}(t)\boldsymbol{K}^{-1}(t)
$
and the fact that
$
    \beta \boldsymbol{I} + \gamma\boldsymbol{K}(t) - \delta\boldsymbol{K}^2(t) \ge 0, \forall t
$
yields
$
    \dot{V} \le \alpha \frac{\left(\boldsymbol{\phi}^T\tilde{\boldsymbol{\theta}}\right)^T\left(\boldsymbol{\phi}^T\tilde{\boldsymbol{\theta}}\right)}{n_s^2} + \dot{\tilde{\boldsymbol{\theta}}}^T\boldsymbol{K}^{-1}\tilde{\boldsymbol{\theta}} + \tilde{\boldsymbol{\theta}}\boldsymbol{K}^{-1}\tilde{\boldsymbol{\theta}}
$.
Using the fact that $\boldsymbol{\theta}^*$ is constant and projection can only make $\dot{V}$ more negative \cite{ioannou2012robust} implies
$
    \tilde{\boldsymbol{\theta}}^T \boldsymbol{K}^{-1} \dot{\tilde{\boldsymbol{\theta}}}
    =
    \tilde{\boldsymbol{\theta}}^T \boldsymbol{K}^{-1} 
    \dot{\hat{\boldsymbol{\theta}}}
    \le
    \alpha(\boldsymbol{\phi}^T\tilde{\boldsymbol{\theta}})^T\boldsymbol{\epsilon}
$.
Then
$
    \dot{V} 
    \le
    \alpha
    (\frac{(\boldsymbol{\phi}^T\tilde{\boldsymbol{\theta}})^T(\boldsymbol{\phi}^T\tilde{\boldsymbol{\theta}})}{n_s^2}
    +
    2(\boldsymbol{\phi}^T\tilde{\boldsymbol{\theta}})^T\boldsymbol{\epsilon}
    )
$.
Next, use equation~\eqref{eq: para estimation eq} and \eqref{eq: estimation error} to obtain
$
    n_s^2(t)\boldsymbol{\epsilon}(t) = -\boldsymbol{\phi}^T(t)\tilde{\boldsymbol{\theta}}(t)
$
and
$
    \dot{V} 
    \le
    \alpha
    \left(
    \frac{n_s^4(t)|\boldsymbol{\epsilon}(t)|^2}{n_s^2}
    -
    2n_s^2(t)|\boldsymbol{\epsilon}|^2
    \right)
$
which may be simplified to
$
    \dot{V} \le - \alpha n_s^2(t) |\boldsymbol{\epsilon}(t)|^2
$.
Thus $\alpha n_s^2(t) \boldsymbol{\epsilon}(t) \to 0$ as $t \to \infty$.
The projection operation in (\ref{eq: parameter update}) can only make $|\dot{\hat{\boldsymbol{\theta}}}|$ smaller, so
    $|\dot{\hat{\boldsymbol{\theta}}}|
    \le
    \frac{\|\boldsymbol{K}(t)\boldsymbol{\phi}(t)\|}{n_s^2(t)}
    \alpha n_s^2(t) |\boldsymbol{\epsilon}(t)|$
and $|\dot{\hat{\boldsymbol{\theta}}}| \to 0$ as $t \to \infty$ because $\frac{\|\boldsymbol{K}(t)\boldsymbol{\phi}(t)\|}{n_s^2(t)}$ is uniformly bounded.
As a result, the term $\boldsymbol{\varepsilon}(t)$, which is proportional to $|\dot{\hat{\boldsymbol{\theta}}}|$, and the measured output $\boldsymbol{y}(t)$ converge to zero as $t \to \infty$.

\subsection{Estimating the cost of the next step}
\label{sec: app est Lkp1}
This section details the method to efficiently estimate the cost function $\mathcal{L}(k+1)$ of the next step when $t\in[t_k, t_{k+1}]$.
We first observe that the analytical solution of equation \eqref{eq: cmd. conti.} takes the form of 
\footnotesize
\begin{align}
    x_{sc}^{c}(t) =
    m_k e^{\lambda(t-t_k)} 
    + n_k e^{-\lambda(t-t_k)}
    \label{eq: x_sc^c}
\end{align}
\normalsize
where
$\lambda = \sqrt{g/z_{sc}}$, and $m_k, n_k$ are the coefficients of the commanded trajectory for the $k^{\text{th}}$ step.
$\boldsymbol{z}(k) = [m_k, n_k]^T$ can be used to represent the commanded trajectory for the $k^{\text{th}}$ step. 
For all time $t\in \mathbb{R}$, an estimated $\hat{\boldsymbol{z}}_{\boldsymbol{p}(t)}(k)$ can be obtained from the current state $\boldsymbol{p}(t)$ at time $t$.
Given $\hat{\boldsymbol{z}}_{\boldsymbol{p}(t)}(k)$ and equation \eqref{eq: x_sc^c}, the estimated pre-impact CoM state $\hat{\boldsymbol{p}}^c(t_{k+1}^-)$ can be computed.
In addition, the current walking output $\boldsymbol{y}(k) = [v_{\text{walk}}(k), o_{\text{walk}}(k)]^T$ for the $k^{\text{th}}$ step can be described as a linear transformation of $\boldsymbol{z}(k)$
\footnotesize
\begin{align}
    \boldsymbol{y}(k)
     = 
    \boldsymbol{G} 
    \boldsymbol{z}(k)
    =
    \begin{bmatrix}
        \frac{e^{\lambda T}-1}{T} 
        &
        \frac{e^{-\lambda T}-1}{T} 
        \\
        \frac{e^{\lambda T}+1}{2} 
        &
        \frac{e^{-\lambda T}+1}{2} 
    \end{bmatrix}
    \boldsymbol{z}(k)
    \label{eq: Gz}
\end{align}
\normalsize

To predict the cost function of the next step, the discrete event~\eqref{eq: xc update} needs to be considered. 
As $\boldsymbol{p}(t_{k+1}^+)$ remains unknown until the post-impact instant, it is estimated by 
$\boldsymbol{p}(t_{k+1}^+)
\approx
\hat{\boldsymbol{p}}(t_{k+1}^+)
=
\hat{\boldsymbol{p}}^c(t_{k+1}^-)
+
\begin{bmatrix}
    -s(t_{k+1}^-) \\
    0
\end{bmatrix}
+
\begin{bmatrix}
    s_{ada}(t_k^-) \\
    r_{ada}(t_k^-)
\end{bmatrix}$
where $s(t_{k+1}^-)$ is the newest solution from the footstep planner and $s_{ada}(t_k^-), r_{ada}(t_k^-)$ are the newest control outputs from the adaptive controller.
This approximation can be expressed as a linear transformation between the current estimated coefficient $\hat{\boldsymbol{z}}_{\boldsymbol{p}(t)}(k)$ and the estimated coefficient of the next step $\hat{\boldsymbol{z}}(k+1)$:
\footnotesize
\begin{align}
    \boldsymbol{L} \boldsymbol{z}(k+1)
    =
    \boldsymbol{L} \boldsymbol{E} \hat{\boldsymbol{z}}_{\boldsymbol{p}(t)}(k) 
    +
    \boldsymbol{\Delta}_z(k)
    \label{eq: z update}
\end{align}
\normalsize
where
\footnotesize
$   \boldsymbol{L} = \begin{bmatrix}
        1 & 1 \\ \lambda & - \lambda
    \end{bmatrix}
    , 
    \boldsymbol{E} = \begin{bmatrix}
        e^{\lambda T} & 0 \\
        0 & e^{- \lambda T}
    \end{bmatrix}
$, 
$
    \boldsymbol{\Delta}_z(k) 
    =
    \begin{bmatrix}
        s_{ada} \\ r_{ada}
    \end{bmatrix}
    +
    \begin{bmatrix}
        -s(t_k^-) \\ 0
    \end{bmatrix}
$.
\normalsize
By equations \eqref{eq: Gz}, \eqref{eq: z update}, and the definition of $e_{walk}(k)$, the next step's output error in walking behavior can be efficiently estimated by
$\boldsymbol{e}_{walk}(k+1)
    =
    \mathbf{A}_{e}\boldsymbol{e}_{walk}(k)
    +
    \mathbf{B}_{e} [ \mathbf{C}_e(k) + \begin{bmatrix}
        -s(t_k^-) \\ 0
    \end{bmatrix} ]$
where
$\boldsymbol{A}_e = \boldsymbol{GEG}^{-1}$,
$\boldsymbol{B}_{e} = \boldsymbol{GL}^{-1}$, 
$\boldsymbol{C}_e(k) = - \boldsymbol{LEG}^{-1} \boldsymbol{y}^{d}(k)
+ \boldsymbol{LG}^{-1} \boldsymbol{y}^{d}(k+1) + [s_{ada}, r_{ada}]^T$, and
$\boldsymbol{y}^{d}(k)$ is the desired walking behavior $\boldsymbol{y}(k)$ at the $k^{\text{th}}$ step.

\bibliographystyle{IEEEtran}
\bibliography{Reference.bib}


\vspace{11pt}




\vfill

\end{document}